%% file: main.tex
\documentclass{article}

\usepackage{arxiv}
\usepackage[numbers,compress]{natbib}
\usepackage[dvipsnames]{xcolor}

\usepackage[utf8]{inputenc} 
\usepackage[T1]{fontenc}    
\usepackage{hyperref}       
\usepackage{url}            
\usepackage{booktabs}       
\usepackage{amsfonts}  
\usepackage{amsmath}
\usepackage{amsthm}
\newtheorem*{remark}{Remark}
\newtheorem{proposition}{Proposition}
\newtheorem{definition}{Definition}
\newtheorem{theorem}{Theorem}
\newtheorem{lemma}{Lemma}
\usepackage{nicefrac}       
\usepackage{microtype}      
\usepackage{xcolor}         
\usepackage{graphicx}       
\usepackage{subcaption}     
\usepackage{bbm}
\usepackage{paralist}
\usepackage{microtype}
\usepackage{bbm}
\usepackage{amssymb}
\usepackage{pdfpages}
\usepackage{fancyhdr}
\usepackage{graphicx}
\usepackage{caption}
\usepackage{subcaption}
\usepackage{multirow}
\usepackage{nicefrac}       
\usepackage{stmaryrd}
\usepackage[parfill]{parskip}
\usepackage{mathtools}
\usepackage{cases}
\usepackage{comment}
\usepackage{paralist}
\usepackage{breqn}
\usepackage{color}
\usepackage{algorithm, algorithmic}
\usepackage{appendix}
\usepackage{xspace}
\usepackage{enumitem}
\usepackage{gensymb}
\usepackage[english]{babel}
\usepackage{xcolor}
\usepackage{todonotes}
\usepackage{wrapfig}
\usepackage{letltxmacro}
\usepackage{breqn}

\include{notation}

\usepackage{hyperref}       
\usepackage{url}            

\hypersetup{
colorlinks = true,
linkcolor = Blue,
anchorcolor = blue,
citecolor = Blue,
filecolor = cyan,
menucolor = ForestGreen,
runcolor = cyan,
urlcolor = ForestGreen}

\makeatletter
\newtheorem*{rep@theorem}{\rep@title}
\newcommand{\newreptheorem}[2]{%
\newenvironment{rep#1}[1]{%
 \def\rep@title{#2 \ref{##1}}%
 \begin{rep@theorem}}%
 {\end{rep@theorem}}}
\makeatother

\newreptheorem{theorem}{Theorem}
\newreptheorem{lemma}{Lemma}
\newenvironment{hproof}{%
  \proof}{\endproof}

\newcommand{\xhdr}[1]{\vspace{0em}\noindent{{\bf #1.}}}
\newcommand{\ie}{\textit{i.e., }}
\newcommand{\eg}{\textit{e.g., }}

\newcommand{\scfe}{{SCFE}\xspace}
\newcommand{\cw}{{C\&W}\xspace}
\newcommand{\chvae}{{C-CHVAE}\xspace}
\newcommand{\deepfool}{{DeepFool}\xspace}
\newcommand{\zhao}{{NAE}\xspace}
\newcommand{\counter}{counterfactual\xspace}
\newcommand{\counters}{counterfactuals\xspace}

\newcommand{\Counters}{Counterfactuals\xspace}
\newcommand{\counterexp}{counterfactual explanation\xspace}
\newcommand{\counterexps}{counterfactual explanations\xspace}

\newcommand{\Counterexps}{Counterfactual explanations\xspace}
\newcommand{\advex}{adversarial example\xspace}
\newcommand{\advexs}{adversarial examples\xspace}

\newcommand{\Advexs}{Adversarial examples\xspace}
\newcommand{\counterexpmethod}{counterfactual explanation method\xspace}
\newcommand{\counterexpmethods}{counterfactual explanation methods\xspace}

\newcommand{\Counterexpmethods}{Counterfactual explanation methods\xspace}
\newcommand{\advexmethod}{adversarial example generation method\xspace}
\newcommand{\advexmethods}{adversarial example generation methods\xspace}

\newcommand{\Advexmethods}{Adversarial example generation methods\xspace}

\DeclareMathOperator*{\argmin}{arg\,min}
\newif\ifcomments
\ifcomments
    \providecommand{\martin}[2][]{{\protect\color{magenta}{[\textbf{M}:\textbf{#1} #2]}}}
    \providecommand{\hima}[2][]{{\protect\color{red}{[\textbf{H}:\textbf{#1} #2]}}}
    \providecommand{\sohini}[2][]{{\protect\color{violet}{[\textbf{Soh}:\textbf{#1} #2]}}}
    \providecommand{\sj}[2][]{{\protect\color{blue}{[\textbf{SJ}:\textbf{#1} #2]}}}
    \providecommand{\chirag}[2][]{{\protect\color{brown}{[\textbf{C}:\textbf{#1} #2]}}}
\else
    \providecommand{\martin}[2][]{}
     \providecommand{\hima}[2][]{}
     \providecommand{\sohini}[2][]{}
      \providecommand{\sj}[2][]{}
     \providecommand{\chirag}[2][]{}
\fi

\title{Exploring Counterfactual Explanations Through the Lens of Adversarial Examples: A Theoretical and Empirical Analysis}

\author{%
  Martin Pawelczyk \\
  University of Tübingen \\
  \texttt{martin.pawelczyk@uni-tuebingen.de} \\
   \And
      Chirag Agarwal \\
   Harvard University \\
   \texttt{chirag\_agarwal@hms.harvard.edu} \\
   \And 
   Shalmali Joshi \\
  Harvard University \\
   \texttt{shalmali@seas.harvard.edu} \\
   \And
   Sohini Upadhyay \\
   Harvard University \\
   \texttt{supadhyay@g.harvard.edu} \\
   \And
   Himabindu Lakkaraju \\
   Harvard University \\
   \texttt{hlakkaraju@hbs.edu} \\
}

\begin{document}

\maketitle

\input{000abstract}

\input{010intro}

\input{020related}

\input{030prelims}

\input{040methods}

\input{050experiments}
\input{060conclusion}

\bibliographystyle{plainnat}
\bibliography{main.bib}

\clearpage
\appendix
\input{appendix}

\end{document}

%% file: notation.tex
\newcommand{\thref}[1]{{Theorem}~\ref{#1}}

\newcommand{\lemref}[1]{{Lemma}~\ref{#1}}

\def\to{{\,\rightarrow\,}}

\mathchardef\mhyphen="2D

\newcommand{\vertiii}[1]{{\left\vert\kern-0.25ex\left\vert\kern-0.25ex\left\vert #1
    \right\vert\kern-0.25ex\right\vert\kern-0.25ex\right\vert}}

\def\bw{{\mathbf{w}}}
\def\bx{{\mathbf{x}}}

\def\bz{{\mathbf{z}}}

\def\bI{{\mathbf{I}}}

\def\cG{\mathcal{G}}

\def\cI{\mathcal{I}}

\def\cN{\mathcal{N}}

\def\cX{\mathcal{X}}
\def\cY{\mathcal{Y}}
\def\cZ{\mathcal{Z}}



%% file: 000abstract.tex
\begin{abstract}
\looseness=-1

As machine learning (ML) models become more widely deployed in high-stakes applications, \counterexps have emerged as key tools for providing actionable model explanations in practice. Despite the growing popularity of counterfactual explanations, a deeper understanding of these explanations is still lacking. In this work, we systematically analyze \counterexps through the lens of adversarial examples. We do so by formalizing the similarities between popular \counterexp and \advexmethods identifying conditions when they are equivalent. We then derive the upper bounds on the distances between the solutions output by \counterexp and \advexmethods, which we validate on several real world data sets. By establishing these theoretical and empirical similarities between \counterexps and \advexs, our work raises fundamental questions about the design and development of existing \counterexp algorithms. 

\end{abstract}

%% file: 010intro.tex
\section{Introduction}
With the increasing use of Machine learning (ML) models in critical domains, such as health care and law enforcement, it becomes important to ensure that their decisions are robust and explainable. To this end, several approaches have been proposed in recent literature to explain the complex behavior of ML models~\citep{simonyan2013deep,ribeiro2016should,lundberg2017unified,sundararajan2017axiomatic}. One such popular class of explanations designed to provide recourse to individuals adversely impacted by algorithmic decisions are \emph{\counterexps}~\citep{wachter2017counterfactual,Ustun_2019,barocas2020hidden,venkatasubramanian2020philosophical}. For example, in a credit scoring model where an individual loan application is denied, a \counterexp can highlight the minimal set of changes the individual can make to obtain a positive outcome~\citep{pawelczyk2020learning,karimi2020algorithmic}. Algorithms designed to output counterfactual explanations often attempt to find a closest counterfactual for which the model prediction is positive 
\citep{wachter2017counterfactual,Ustun_2019,pawelczyk2020learning,karimi2020algorithmic}.

\Advexs, on the other hand, were proposed to highlight how vulnerabilities of ML models can be exploited by (malicious) adversaries ~\citep{szegedy2013intriguing,ballet2019imperceptible,cartella2021adversarial}. These \advexs are usually also obtained by finding minimal perturbations to a given data instance such that the model prediction changes \citep{goodfellow2014explaining,carlini2017towards,moosavi2016deepfool}. 


Conceptually, \advexs and \counterexps solve a similar optimization problem~\citep{freiesleben2020counterfactual,wachter2017counterfactual,cartella2021adversarial}. Techniques generating adversarial examples and counterfactual explanations use distance or norm constraints in the objective function to enforce the notion of minimal perturbations. While adversarial methods generate instances that are semantically indistinguishable from the original instance, \counterexps or algorithmic recourse\footnote{Note that \counterexps, contrastive explanations, and recourses are used  interchangeably in prior literature. Counterfactual/contrastive explanations serve as a means to provide recourse to individuals with unfavorable algorithmic decisions. We use these terms interchangeably as introduced and defined by~\citet{wachter2017counterfactual}.}, encourage minimal changes to an input so that so that a user can readily act upon these changes to obtain the desired outcome. In addition, some methods in both lines of work use manifold-based constraints to find natural adversarial examples~\citep{zhao2018generating} or realistic counterfactual explanations by restricting them to lie on the data manifold~\citep{joshi2019towards,pawelczyk2020learning}.

While the rationale of producing a \counter close to the original instance is motivated by the desideratum that \counters should be actionable and easily understandable, producing close instances on the other side of the decision boundary could just as easily indicate adversarial activity. This begs the question to what extent do \counterexp algorithms return solutions that resemble \advexs. However, there has been little to no work on systematically analyzing the aforementioned connections between the literature on \counterexps and \advexs.

\xhdr{Present Work} 
In this work, we approach the study of similarities between \counterexps and \advexs from the perspective of counterfactual explanations for algorithmic recourse. Therefore, we consider consequential decision problems (e.g., loan applications) commonly employed in recourse literature and our choices of data modalities (i.e., tabular data) and algorithms are predominantly motivated by this literature.
In particular, we make one of the first attempts at establishing theoretical and empirical connections between state-of-the-art \counterexp and \advexmethods.  

More specifically, we analyze these similarities by bounding the distances between the solutions of salient \counterexp and popular \advex methods for locally linear approximations.
Our analysis demonstrates that several popular \counterexp and \advexmethods such as the ones proposed by \citet{wachter2017counterfactual} and \citet{carlini2017towards,moosavi2016deepfool},
are equivalent for certain hyperparameter choices.
Moreover, we demonstrate that \chvae and the natural adversarial attack (NAE) \citep{zhao2018generating} provide similar solutions for certain generative model choices.

Finally, we carry out extensive experimentation with multiple synthetic and real-world data sets from diverse domains such as financial lending and criminal justice to validate our theoretical findings. We further probe these methods empirically to validate the similarity between the counterfactuals and adversarial examples output by several state-of-the-art methods. Our results indicate that counterfactuals and adversarial examples output by manifold-based methods such as \zhao and \chvae are more similar compared to those generated by other techniques. By establishing these and other theoretical and empirical similarities, our work raises fundamental questions about the design and development of existing \counterexp algorithms.

%% file: 020related.tex
\section{Related Work}\label{sec:related_work}
This work lies at the intersection of \counterexps and \advexs in machine learning.
Below we discuss related work for each of these topics and their connection.

\xhdr{\Advexs} \Advexs are obtained by making infinitesimal perturbations to input instances such that they force
a ML model to generate adversary-selected outputs. Algorithms designed to successfully generate these examples are called Adversarial attacks~\citep{szegedy2013intriguing, goodfellow2014explaining}. Several attacks have been proposed in recent literature depending on the degree of knowledge/access of the model, training data, and optimization techniques.
While gradient-based methods~\citep{goodfellow2014explaining, kurakin2016adversarial,moosavi2016deepfool} find the minimum $\ell_{p}$-norm perturbations to generate \advexs, generative methods~\citep{zhao2018generating} constrain the search for \advexs to the training data-manifold. Finally, some methods~\citep{cisse2017houdini} generate \advexs for non-differentiable and non-decomposable measures in complex domains such as speech recognition and image segmentation.
We refer to a well-established survey for a more comprehensive overview of \advexs \citep{akhtar2018threat}.

\xhdr{\Counterexps} 
\Counterexpmethods aim to provide explanations for a model prediction in the form of minimal changes to an input instance that changes the original prediction~\citep{wachter2017counterfactual,Ustun_2019,van2019interpretable,karimi2020survey}.
These methods are categorized based on the access to the model (or gradients), sparsity of the generated explanation and whether the generated explanations are constrained to the manifold~\citep{verma2020counterfactual,karimi2020survey}. To this end, \citet{wachter2017counterfactual} proposed a gradient-based method to obtain \counterexps for models using a distance-based penalty and finding the nearest \counterexp. Further, restrictions on attributes such as race, age, and gender are generally enforced to ensure that the output \counterexps are realistic for users to act on them. In addition, manifold-based constraints are imposed in many methods~\citep{pawelczyk2020learning,joshi2019towards} so that the \counterexps are faithful to the data distribution. Finally, causal approaches have recently been proposed to generate \counterexps that adhere to causal constraints~\citep{karimi2020model,barocas2020hidden,karimi2021algorithmic,karimi2020algorithmic}.
\xhdr{Connections between \advexs and \counterexps} Conceptual connections between \advexs and \counterexps have been previously identified in the literature~\citep{freiesleben2020counterfactual,browne2020semantics}. While~\citet{freiesleben2020counterfactual} highlight conceptual differences in aims, formulation and use-cases between both sub-fields suggesting that \counterexps represent a broader class of examples of which \advexs represent a subclass,~\citet{browne2020semantics} focus on discussing the  differences \textit{w.r.t} semantics hidden layer representations of DNNs. Our goal, on the other hand, is to theoretically formalize and empirically analyze the (dis)similarity between these fields.

%% file: 030prelims.tex
\section{Preliminaries}\label{sec:prelim}
\looseness=-1
\xhdr{Notation} 
We denote a classifier $h: \cX{\to}\cY$ mapping features $\bx \in \cX$ to labels $\cY$.
Further, we define $h(\bx){=}g(f(\bx))$, where $f: \cX{\to}\mathbb{R}$ is a scoring function (\eg logits) and $g: \mathbb{R}{\to}\cY$ an activation function that maps output logit scores to discrete labels.
Below we describe some representative methods used in this work to generate \counterexps and \advexs.

\subsection{\Counterexpmethods}
\label{sec:counterexpmethod}
\Counterexps provide recourses by identifying which attributes to change for reversing a models' adverse outcome.
Methods designed to output \counterexps find a \counter $\bx'$ that is "closest" to the original instance $\bx$ and changes the models' prediction $h(\bx')$ to the desired label.
While several of these methods incorporate distance metrics (\eg $\ell_{p}$-norm) or user preferences~\citep{rawal2020interpretable} to find the desired \counters, some efforts also impose causal~\citep{karimi2020algorithmic} or data manifold constraints~\citep{joshi2019towards,pawelczyk2020learning} to find realistic \counters. We now describe \counterexp methods from two broad categories: 1) Gradient-~\citep{wachter2017counterfactual} and 2) search-based~\citep{pawelczyk2020learning}.

\xhdr{Score CounterFactual Explanations (SCFE)}
For a given classifier $h$ and the corresponding scoring function $f$, and a distance function $d: \cX \times \cX \to \mathbb{R}_{+}$, \citet{wachter2017counterfactual} formulate the problem of finding a \counter $\bx'$ for $\bx$ as:
\begin{equation}
    \argmin_{\bx'} ~ \big(f(\bx')-s\big)^{2} + \lambda \, d(\bx,\bx'),
    \label{scfeobj}
\end{equation}
where $s$ is the target score for $\bx$ and $\lambda$ is set to iteratively increase until $f(\bx'){=}s$. More specifically, to arrive at a counterfactual probability of 0.5, the target score for $g(\bx)$ for a sigmoid function is $s{=}0$, where the logit corresponds to a $0.5$ probability for $y{=}1$.


\xhdr{\chvae}\label{subsec:chvae}
\looseness=-1
Let $I_{\gamma}$ and $\cG_{\theta}$ denote the encoder and decoder of the VAE model used by \chvae~\citep{pawelczyk2020learning} to generate realistic \counters. Note that the \counters for $\bx$ are generated in the latent space of the encoder $\cZ$, where $I_{\gamma}: \cX \to \cZ$. Let $\bz$ and $\tilde{\bz}=\bz{+}\delta$ denote the latent representation and generated \counters for the original instance $\bx$. Intuitively, $\cG_{\theta}$ is a generative model that projects the latent \counters to the feature space and $I_{\gamma}$ allows to search for \counters in the data manifold. Thus, the objective function is defined as follows:
\begin{align}
\label{eq:cchvae}
\begin{split}
    & \delta^* = \argmin_{\delta \in \cZ} ~ \| \delta\| \text{ s.t. } h(\cG_{\theta}(I_{\gamma}(\bx^{f})+\delta), \bx^p) \neq h(\bx^{f},\bx^p),
\end{split}
\end{align}
where $\bx^p$ and $\bx^{f}$ indicate the protected and non-protected features of $\bx$ and Eqn.~\ref{eq:cchvae} finds the minimal perturbation $\delta$ by changing the non-protected features $\bx^m$ constrained to the data-manifold.

\subsection{\Advexmethods}\label{sec:adv_ex}
\looseness=-1
Similar to \counterexpmethods, most methods generating \advexs also solve a constrained optimization problem to find perturbations in the input manifold that cause models to misclassify. These methods are broadly categorized into \textit{poisoning}~(e.g., \citet{shafahi2018poison}) and \textit{exploratory}~ (e.g., \citet{goodfellow2014explaining}) methods.
While \textit{poisoning} methods attack the model during training and attempts to learn, influence, and corrupt the underlying training data or model, \textit{Exploratory} methods do not tamper with the underlying model but instead generate specific examples that cause the model to produce the desired output.
Like \counterexpmethods, evasion methods also use gradient-based optimization to generate \advexs. Below, we briefly outline three evasion techniques considered in this work.

\xhdr{\cw Attack}
For a given input $\bx$ and classifier $h(\cdot)$,~\citet{carlini2017towards} formulate the problem of finding an \advex $\bx'{=}\bx{+}\delta$ such that $h(\bx')\neq h(\bx)$ as:
\begin{equation}
    \arg\min_{\bx'} c \cdot \ell(\bx')+d(\bx, \bx')~~\mbox{s.t.}~~\bx' \in [0,1]^{d}
\label{cwobj}
\end{equation}
where $c$ is a hyperparameter and $\ell(\cdot)$ is a loss function such that $h(\bx'){=}y$ if and only if $\ell(\bx')\leq 0$. The constraint $\bx'\in [0,1]^d$ is applied so that the resulting $\bx'$ is within a given data range.

\xhdr{DeepFool}\label{sec:deepfool}
\looseness=-1
For a given instance $\bx$, DeepFool~\citep{moosavi2016deepfool} perturbs $\bx$ by adding small perturbation $\delta_{\text{DF}}$ at each iteration of the algorithm. The perturbations from each iterations are finally aggregated to generate the final perturbation once the output label changes. The minimal perturbation to change the classification model's prediction is the solution to the following objective:
\begin{align}
\begin{split}
    & \delta^{*}_{\text{DF}}(\mathbf{x}) = \argmin_\delta~ ||\delta||_{2} \text{ s.t. } \text{sign}(f(\mathbf{x} + \delta)) \neq \text{sign}(f(\mathbf{x})),
    \end{split}
\end{align}
where $\mathbf{x}$ is the input sample. The closed-form step for each iteration is: $\delta^{*}_{\text{DF}}{=}-(f(\mathbf{x})/||\nabla f(\mathbf{x})||_{2}^{2}) \nabla f(\mathbf{x})$.

\xhdr{Natural Adversarial Example (NAE)} Similar to \chvae, \citet{zhao2018generating} proposes NAE to search for \advexs using a generative model $\cG_{\theta}$ where the similarity is measured in the latent space of $\cG_{\theta}$. Thus, the objective is given by:
\begin{align}\label{eq:zhao}
    \bz^* &= \argmin_{\tilde{\bz} \in \cZ} \| \tilde{\bz} - I_{\gamma}(\bx)\| \text{ s.t. }h(\cG_{\theta}(\tilde{\bz})) \neq h(\bx),
\end{align}
where $I_{\gamma}(\bx)$ corresponds to the latent representation of $\bx$ and $\cG_{\theta}(\tilde{\bz})$ maps the latent sample to the feature space. NAE separately trains an inverter function  from $\cG_{\theta}$ by enforcing the latent representation to be normally distributed (\ie corresponding to the noise model of the generator) while minimizing the reconstruction error of the feature space. 

%% file: 040methods.tex
\section{Theoretical Analysis}\label{sec:compare}
\looseness=-1
In this section, we provide theoretical connections between \counterexp and \advex methods by leveraging similarities in the objective functions and optimization procedures. 
In particular, we compare: 1) \scfe and \cw (Sec.~\ref{subsec:cw_cfe}), 2) \scfe and \deepfool (Sec.~\ref{subsec:df_cfe}), and 3) \chvae and \zhao (Sec.~\ref{subsec:chvae_zhao}) due to their similarity in the objective functions. We do these comparisons either for a specific loss, solutions based on the classification model, or constraints imposed during optimization. We focus on locally linear model approximations as these are often studied as a first step \citep{hardt2016identity,Ustun_2019,rosenfeld2020certified, garreau2020explaining} towards understanding nonlinear model behaviour.



\subsection{\scfe and \cw}\label{subsec:cw_cfe}
Two popular gradient-based methods for generating adversarial and counterfactual samples are the \cw Attack and \scfe, respectively. 
Here, we first show the closed-form solutions for the minimum perturbation required by \cw ($\delta^{*}_{\text{CW}}$) and \scfe ($\delta^{*}_{\text{SCFE}}$) to generate \advexs and \counters. We then leverage these solutions to derive an upper bound for the distance between the adversarial and counterfactual samples. Using the loss function $\ell^{*}(\cdot){=}\max(0,\max_{i}(f(\bx)_{i})-f(\bx)_{y})$ recommended by \citet{carlini2017towards}, we derive an upper bound for the distance between the \counters and \advexs generated using \scfe and \cw. For the upper bound, we first state a lemma that derives the closed-form solution for $\delta^{*}_{\text{SCFE}}$.

\begin{lemma}(Optimal Counterfactual Perturbation)
\label{lemma1} For a scoring function with weights $\bw$ the \scfe method generates a \counter~$\bx_{\text{SCFE}}$~for an input $\bx$ using the \counter perturbation $\delta^{*}_{\text{SCFE}}$~such that:
\begin{equation}
    \delta^{*}_{\text{SCFE}} = m \cdot (\mathbf{\bw \bw^T} + \lambda \mathbf{I})^{-1} \bw,
    \label{SCFE_closedform}
\end{equation}
where $s$ is the target score for $\bx$, $m{=}s{-}f(\bx)$ is the target residual, $f(x){=}\bw^{\top}\bx+b$ is a local linear score approximation, and $\lambda$ is a given hyperparameter.
\end{lemma}
\looseness=-1
\begin{hproof}
We derive the closed-form solution for $\delta^{*}_{\text{SCFE}}$ by formulating the \scfe objective in its vector quadratic form. See Appendix~\ref{app:lemma1} for the complete proof.
\end{hproof}
\begin{remark}
By solving the counterfactual condition $s{:=}f(\bx'){=}f(\delta^*_{\text{SCFE}} + \bx)$ for $\lambda$, we can derive the optimal hyperparameter choice $\lambda^*{=}(\lVert \bw \rVert_2^2)/(\lVert \bw \rVert_2^2-1)$ which leads to $\delta^{**}_{\text{SCFE}}{=}(m/\lVert \bw \rVert_2^2) \cdot \bw.$~We use this result to show equivalence between \scfe, \deepfool and \cw.
\end{remark}

Using Lemma~\ref{lemma1}, we now formally state and derive the upper bound for the distance between the \counters and \advexs.
\begin{theorem}(Difference between \scfe and \cw)
\label{thm1}
Under the same conditions as stated in Lemma \ref{lemma1},
the normed difference between the \scfe \counter $\mathbf{x}_{\text{SCFE}}$ and \cw \advex $\mathbf{x}_{\text{CW}}$ using the loss function $\ell^{*}(\cdot)$ is upper bounded by:
\begin{align}
\begin{split}
    & \lVert\mathbf{x}_{\text{SCFE}} - \mathbf{x}_{\text{CW}} \rVert_{p} \leq \biggl \lVert \frac{1}{\lambda} \bigg(\mathbf{I} - \frac{\mathbf{w}\mathbf{w}^{T}}{\lambda+\lVert \mathbf{w}||_{2}^{2}}\bigg) \big(s-f(\mathbf{x})\big) - c\mathbf{I} \biggr \rVert_{p} ||\mathbf{w}||_{p}.
\end{split}
    \label{eq:w_cw_bound}
\end{align}
\end{theorem}
\begin{hproof}
We first derive the closed-form solution for the perturbation used by \cw. Intuitively, this solution is equivalent to shifting $\bx$ in the direction of the models' decision boundary scaled by $c$. The upper bound follows by applying Lemma~\ref{lemma1} and Cauchy-Schwartz inequality. Moreover, choosing the optimal hyperparameter $\lambda^*{=}(\lVert \bw \rVert_2^2)/(\lVert \bw \rVert_2^2-1)$ and setting $c{=}m/\lVert \bw \rVert_2^2$ yields equivalence, \ie $||\mathbf{x}_{\text{SCFE}} - \mathbf{x}_{\text{CW}}||_{p}=0$. See Appendix~\ref{app:thm1} for the complete proof.
\end{hproof}
We note that the upper bound is smaller when the original score $f(\bx)$ is close to the target score $s$, suggesting that $\bx_{\text{SCFE}}$ and $\bx_{\text{CW}}$ are more similar when $\bx$ is closer to the decision boundary.
\subsection{\scfe and \deepfool}\label{subsec:df_cfe}
\deepfool is an adversarial attack that uses an iterative gradient-based optimization approach to generate \advexs. Despite the differences in the formulations of \scfe and \deepfool, our theoretical analysis reveals a striking similarity between the two methods.
In particular, we provide an upper bound for the distance between the solutions output by \counters and \advexs generated using \scfe and \deepfool, respectively.

\begin{theorem}(Difference between \scfe and \deepfool)
\label{thm2}
Under the same conditions as stated in Lemma \ref{lemma1}, the normed difference between the \scfe \counter $\mathbf{x}_{\text{SCFE}}$ and the \deepfool \advex $\mathbf{x}_{\text{DF}}$ is upper bounded by:
\begin{align}
\begin{split}
    & \lVert \mathbf{x}_{\text{SCFE}} - \mathbf{x}_{\text{DF}} \rVert_{p} \leq \biggl \lVert \bigg(\mathbf{I} - \frac{\mathbf{w}\mathbf{w}^{T}}{\lambda+ \lVert \mathbf{w}||_{2}^{2}} \bigg) \frac{\big(s-f(\mathbf{x})\big)}{\lambda} + \bI \frac{f(\mathbf{x})}{\lVert \mathbf{w} \rVert_{2}^{2}} \biggr \rVert_{p} \lVert \mathbf{w} \rVert_{p}.
    \label{eq:w_df_bound}
\end{split}
\end{align}
\end{theorem}
\begin{hproof}
\looseness=-1
We show the similarity between \scfe and \deepfool methods by comparing their closed-form solutions for the generated \counter and \advexs. Similar to Theorem~\ref{thm1}, the results follow from Cauchy-Schwartz inequality, (see Appendix~\ref{app:thm2} for the complete proof). Moreover, choosing the optimal hyperparameter $\lambda^*=(\lVert \bw \rVert_2^2)/(\lVert \bw \rVert_2^2-1)$  and setting $s{:=}0$ yields equivalence, \ie $||\mathbf{x}_{\text{SCFE}} - \mathbf{x}_{\text{DF}}||_{p}=0$.
\end{hproof}
The right term in the inequality (Eqn.~\ref{eq:w_df_bound}) entails that the $l_{p}$-norm of the difference between the generated samples is bounded if: 1) the predicted score is closer to the target score of a given input, and 2) the gradients with respect to the logit scores of the underlying model are bounded.

\subsection{Manifold-based methods}\label{subsec:chvae_zhao}
We formalize the connection between manifold-based methods by comparing \zhao to \chvae as both rely on generative models. While \chvae uses variational autoencoders, NAE uses GANs, specifically Wasserstein GAN~\citep{arjovsky2017wasserstein}, to generate \advex. To allow a fair comparison, we assume that both methods use the same generator $\cG_{\theta}$ and inverter $I_{\gamma}$ networks.

\begin{proposition}
Let $p{=}\emptyset$ in \chvae. Assuming that \chvae and \zhao use the same generator $\cG_{\theta}$ and inverter functions $\cI_{\theta}$. Then the proposed objectives of \zhao and \chvae are equivalent.
\end{proposition}
\begin{proof}
Since $p{=}\emptyset$, \eqref{eq:cchvae} reduces to:
\begin{equation}\label{eq:cchvae2}
    \delta^* {=} \argmin_{\mathclap{\delta \in \cZ}} \| \delta \| \text{ s.t. }h(\cG_{\theta}(I_{\gamma}(\bx^f)+\delta)) \neq h(\bx^f)
\end{equation}
Also, $I_{\gamma}(\bx){=}\bz$. Replacing $\tilde{\bz}{-}\bz{=}\delta$ in eqn.\ \ref{eq:zhao}, we get:
\begin{equation}\label{eq:zhao2}
    \delta^* = \argmin_{\delta \in \cZ} \| \delta\|
     \text{ s.t. } h(\cG_{\theta}(I_{\gamma}(\bx) + \delta)) \neq h(\bx)
\end{equation}
Since $\bx^f{=}\bx$, we get the equivalence.
\end{proof}
Both \chvae and \zhao use search methods to generate \advexs or \counters using the above objective function. In particular, both \zhao and \chvae samples $\bz$ using an $\ell_p$-norm ball of radius range $(r_{\text{NAE}}, \Delta r_{\text{NAE}}]$ and $r_{\text{C}}$. $\tilde{\bz}_{\text{NAE}}$ denotes the solution returned by~\citet{zhao2018generating} and $\tilde{\bz}_{\text{C}}$ the solution returned by \chvae. We denote $r_{\text{NAE}}^*$ and $r_{\text{C}}^*$ as the corresponding radius parameters from \zhao and \chvae, respectively, and restrict our analysis to the class of $L$-Lipschitz generative models:
\begin{definition}~\citet{bora2017compressed}:
A generative model $\cG_{\theta}(\cdot)$ is $L$-Lipschitz if~$\forall\,~\bz_{1},~\bz_{2} \in \cZ$, we have,
 \begin{equation}
     \|\cG_{\theta}(\bz_1) - \cG_{\theta}(\bz_2)\|_p \leq L \|\bz_1 - \bz_2\|_p.
 \end{equation}
\end{definition}
Note that commonly used DNN models comprise of linear, convolutional and activation layers, which satisfy Lipschitz continuity~\citep{gouk2021regularisation}.
\begin{lemma}(Difference between \chvae and \zhao)
\label{thm3}
Let $\tilde{\bz}_{\text{C}}$ and $\tilde{\bz}_{\text{NAE}}$ be the output generated by \chvae and \zhao by sampling from $\ell_p$-norm ball in the latent space using an $L$-Lipschitz generative model $\cG_{\theta}(\cdot)$. Analogously, let $\bx_{\text{NAE}}{=}\cG_{\theta}(\tilde{\bz}_{\text{NAE}})$ and $\bx_{\text{C}}{=}\cG_{\theta}(\tilde{\bz}_{\text{C}})$ generate perturbed samples by design of the two methods. Let $r_{\text{NAE}}^*$ and $r_{\text{C}}^*$ be the corresponding radii chosen by each algorithm such that they successfully return an \advex or \counter.   
Then, $\|\bx_{\text{C}} - \bx_{\text{NAE}}\| \leq L(r_{\text{C}}^* + r_{\text{NAE}}^*)$.
\end{lemma}
\begin{hproof}
\looseness=-1
The proof follows from triangle inequality, $L$-Lipschitzness of the generative model, and the fact that the $\ell_p$-norm of the method's outputs are known in the latent space. See Appendix~\ref{app:thm3} for a detailed proof.
\end{hproof}
\looseness=-1
Intuitively, the \advex and \counterexp generated by the methods are bounded depending on the data manifold properties (captured by the Lipschitzness of the generative model) and the radius hyperparameters used by the search algorithms.

%% file: 050experiments.tex
\begin{figure*}[tb]
    \centering
\begin{subfigure}[b]{0.32\textwidth}
         \centering
         \includegraphics[width=\textwidth]{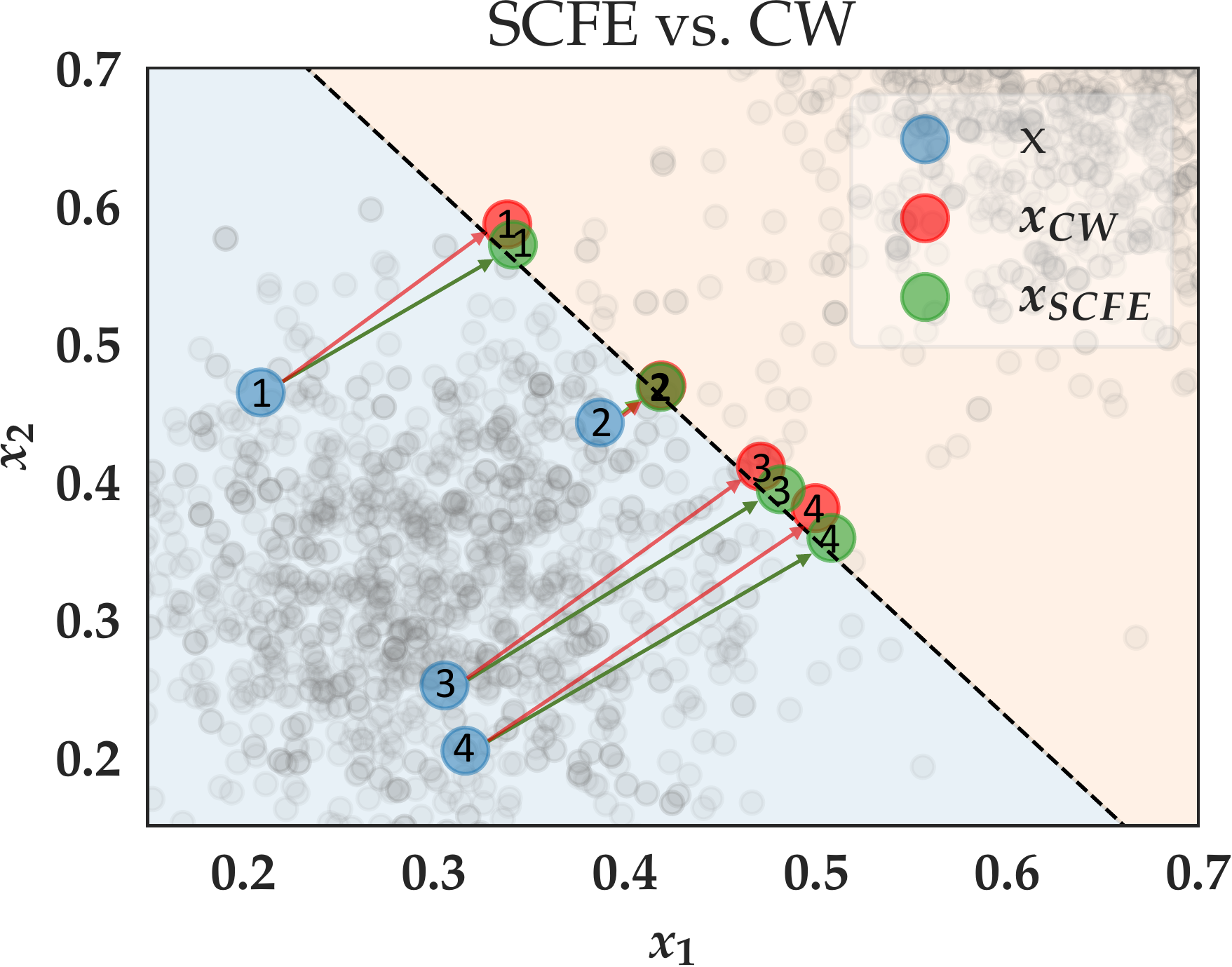}
\end{subfigure}
\begin{subfigure}[b]{0.32\textwidth}
         \centering
         \includegraphics[width=\textwidth]{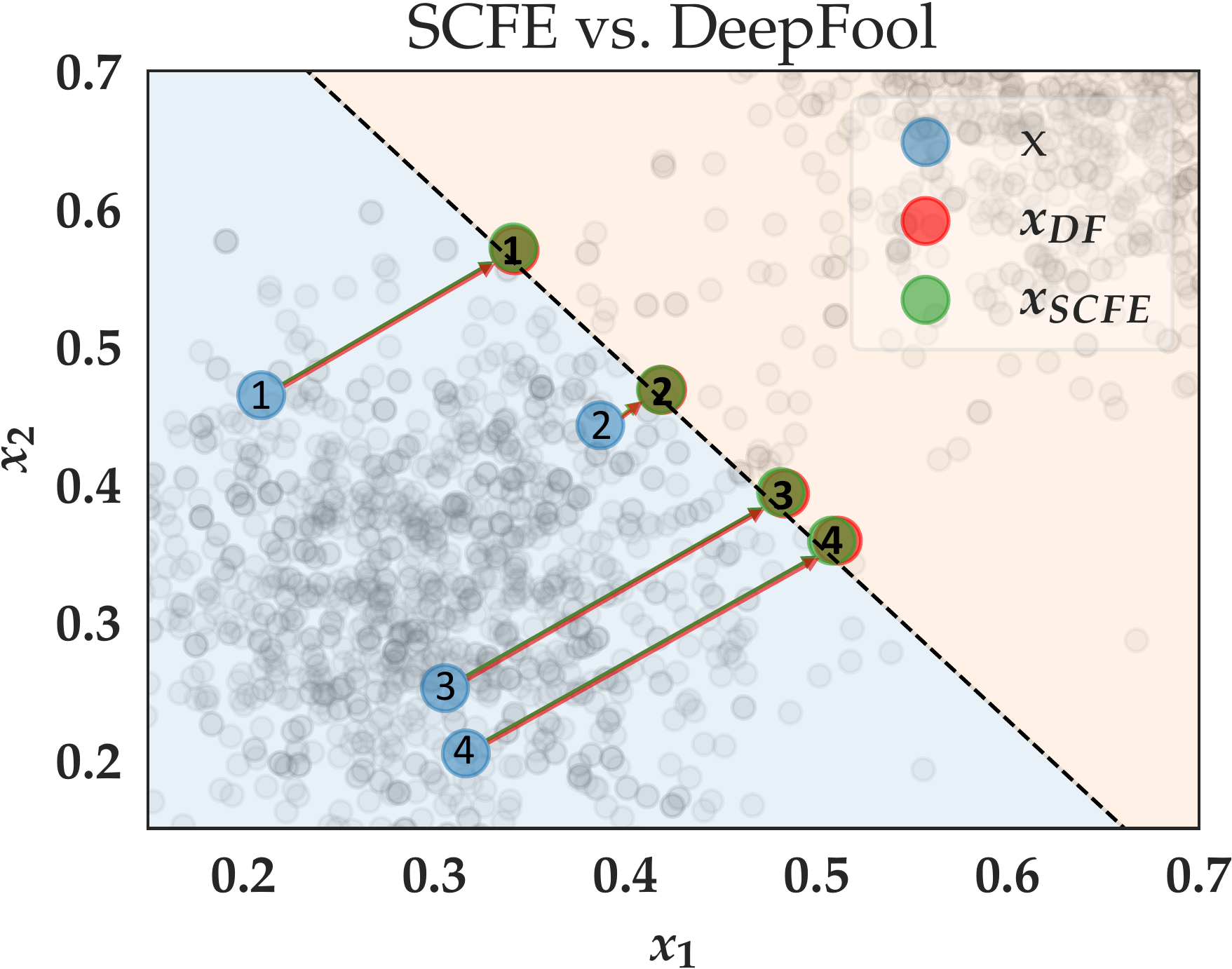}
\end{subfigure}
\begin{subfigure}[b]{0.32\textwidth}
         \centering
         \includegraphics[width=\textwidth]{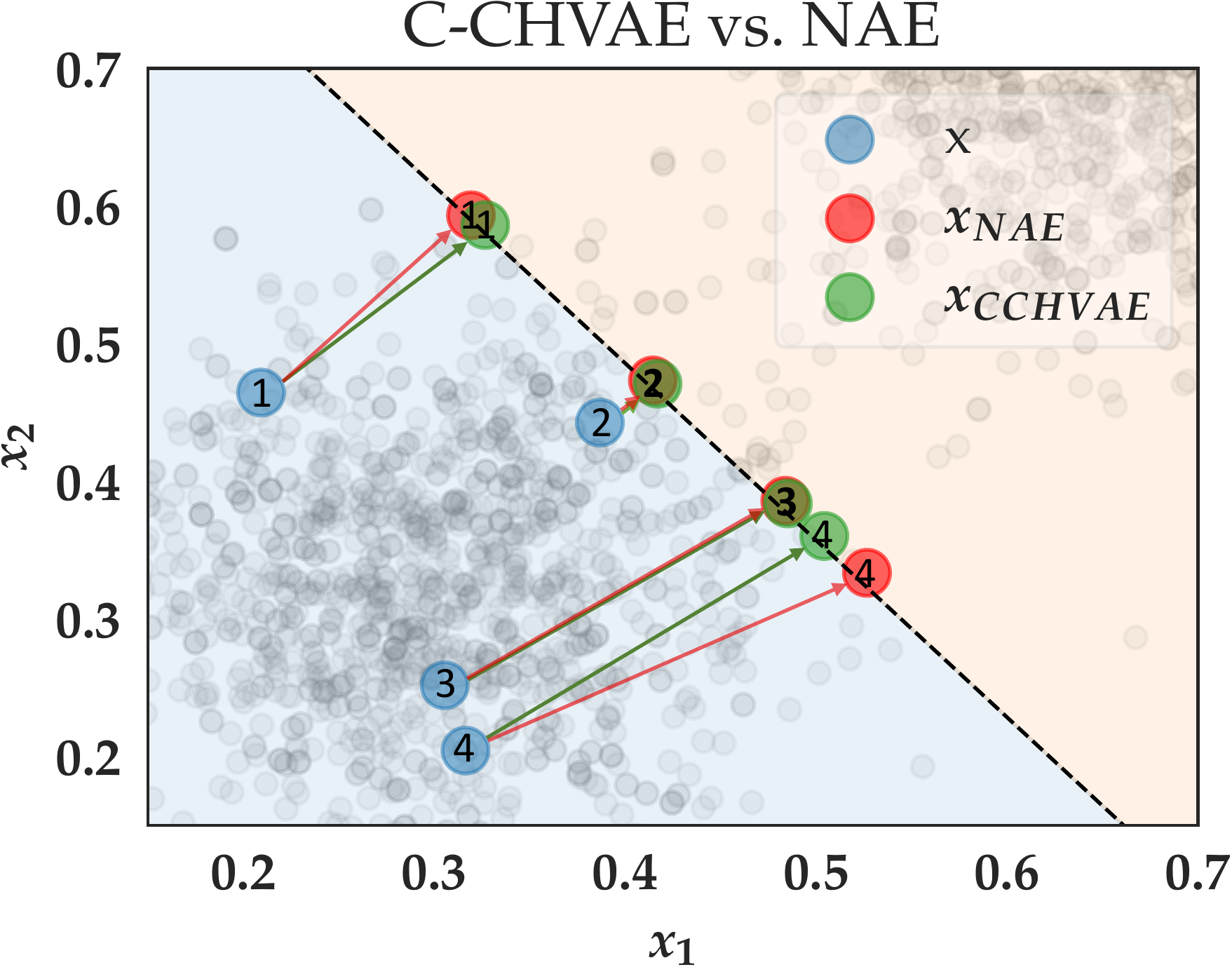}
\end{subfigure}
    \caption{ Similarity comparison of \advex and \counterexp methods. Based on synthetic data, we generate \advexs (in red) and \counterexps (in green) for some randomly chosen test set points (in blue) using methods described in Sec.~\ref{sec:prelim}. (\textbf{Left}) Both \scfe (in green) and \cw (in red) samples are close to each other, indicating strong similarity between these methods. (\textbf{Middle}) \scfe (in green) and \deepfool (in red) samples exactly coincide, indicating equivalence. (\textbf{Right}) \chvae (in green) and \zhao (in red) samples are closer if the blue factual points are closer to the boundary.
    }
    \label{fig:synthetic_2d}
\end{figure*}

\section{Experimental Evaluation}\label{sec:results}
\looseness=-1
We now present the empirical analysis to demonstrate the similarities between \counterexps and \advexs. More specifically, we verify the validity of our theoretical upper bounds using real-world datasets and determine the extent to which \counterexps and \advexs similar to each other.



\subsection{Experimental Setup}\label{sec:setup}
We first describe the synthetic and real-world datasets used to study the connections between \counterexps and \advexs, and then we outline our experimental setup.

\xhdr{Synthetic Data} We generate $5000$ samples from a mixture of Gaussians with pdfs $\cN(\mu_1{=}[1.0,1.0], \Sigma_1{=}\textbf{I})$ and $\cN(\mu_2{=}[-1.0, -1.0], \Sigma_2{=}\textbf{I})$. 

\xhdr{Real-world Data} We use three datasets in our experiments. 1) The \textit{UCI Adult} dataset \citep{Dua:2019} consisting of $48842$ individuals with demographic (\eg age, race, and gender), education (degree), employment (occupation, hours-per-week), personal (marital status, relationship), and financial (capital gain/loss) features. The task is to predict whether an individual's income exceeds $\$50\text{K}$ per year or not. 2) The \textit{COMPAS} dataset \citep{larson2016we} comprising of 10000 individuals representing defendants released on bail. The task is to predict whether to release a defendant on bail or not using features, such as criminal history, jail, prison time, and defendant's demographics. 3) The \textit{German Credit} dataset from the UCI repository~\citep{Dua:2019} consisting of demographic (age, gender), personal (marital status), and financial (income, credit duration) features from $1000$ credit applications. The task is to predict whether an applicant qualifies for credit or not.

\looseness=-1
\xhdr{Methods} Following our analysis in Sec.~\ref{sec:compare}, we compare the following pair of methods: i) \scfe \citep{wachter2017counterfactual} vs.\ \cw \citep{carlini2017towards}, ii) \scfe vs.\ \deepfool \citep{moosavi2016deepfool}, and iii) \chvae \citep{pawelczyk2020learning} vs.\ \zhao \citep{zhao2018generating}. 

\xhdr{Prediction Models} For the synthetic dataset, we train a logistic regression model (LR) to learn the mixture component (samples and corresponding decision boundary shown in Fig.~\ref{fig:synthetic_2d}), whereas for real-world datasets, we obtain \advexs and \counters using LR and artificial neural network (ANN) models. See Appendix~\ref{app:expts} for more details.


\looseness=-1
\textbf{Implementation Details} 
For all real-world data, \advexs and \counters are generated so as to flip the target prediction label from unfavorable ($y{=}0$) to favorable ($y{=}1$). We use $\ell_2$-norm as the distance function in all our experiments. 
We partition the dataset into train-test splits where the training set is used to train the predictor models.
\Advexs and \counters are generated for the trained models using samples in the test splits. For \counterexpmethods applied to generate recourse, all features are assumed actionable for fair comparison with \advexmethods. 
See Appendix~\ref{app:expts} for more implementation details.

\subsection{Results}
\xhdr{Validating our Theoretical Upper Bounds}
\label{res:theory_val} We empirically validate the theoretical upper bounds obtained in Sec.~\ref{sec:compare}. To this end, we first estimate the bounds for each instance in the test set according to Theorems~\ref{thm1} and \ref{thm2}, and compare them with the empirical estimates of the $\ell_2$-norm differences (LHS of Theorems 1 and 2). 
We use the same procedure to validate the bounds from Lemma \ref{lemma2}.

\looseness=-1
\xhdr{\scfe~vs.\ \cw and \deepfool} In Fig.~\ref{fig:boundsall}, we show the empirical evaluation of our theoretical bounds for all real-world datasets. For each dataset, we show four box-plots: empirical estimates (green) and theoretical upper bounds (blue) of the distance ($\ell_2$-norm) between the resulting \counters and \advexs for \scfe and \cw (labeled as SCFE vs. CW), and \scfe and \deepfool (labeled as SCFE vs. DF). Across all three datasets, we observe that no bounds were violated for both theorems. The gap between empirical and theoretical values is relatively small for German credit dataset as compared to COMPAS and Adult datasets.
From Theorems~\ref{thm1} and \ref{thm2}, we see that the bound strongly depends on the norm of the logit score gradient $w{=}\nabla_x f(\bx)$, \eg for Adult dataset these norms are relatively higher leading to less tight bounds.

\xhdr{\chvae vs.\zhao} In Fig.~\ref{fig:bound_logistic_lem2}, we validate the bounds obtained in \lemref{lemma2} for all three datasets using an encoder-decoder framework. We observe that our upper bounds are tight, thus validating our theoretical analysis for comparing manifold-based \counterexp (\chvae) and \advexmethod (\zhao).
\begin{figure}[!t]
\centering
\vspace{-0.1in}
\begin{subfigure}[b]{0.32\columnwidth}
         \centering
         \includegraphics[width=\textwidth]{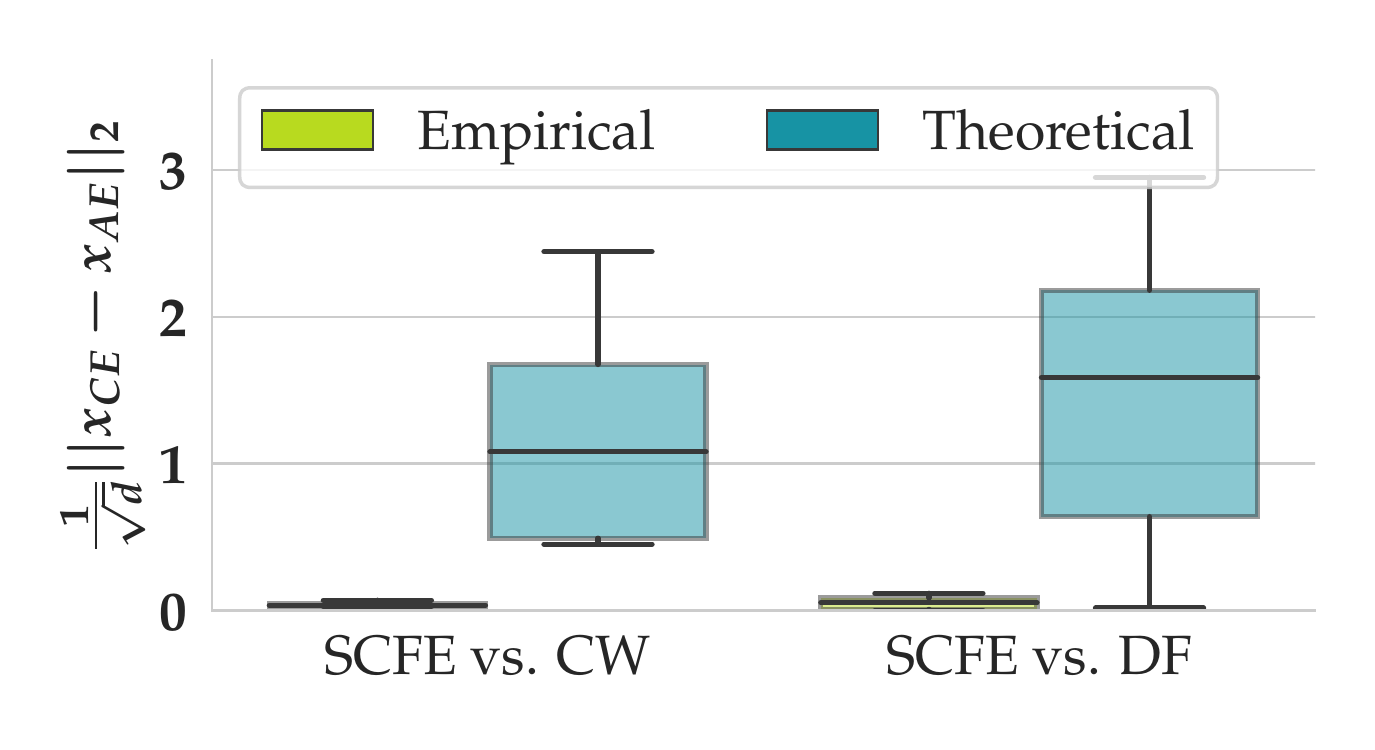}
         \caption{Adult -- LR}
         \label{fig:bound_adult_logistic}
\end{subfigure}
\hfill 
\begin{subfigure}[b]{0.32\columnwidth}
         \centering
         \includegraphics[width=\textwidth]{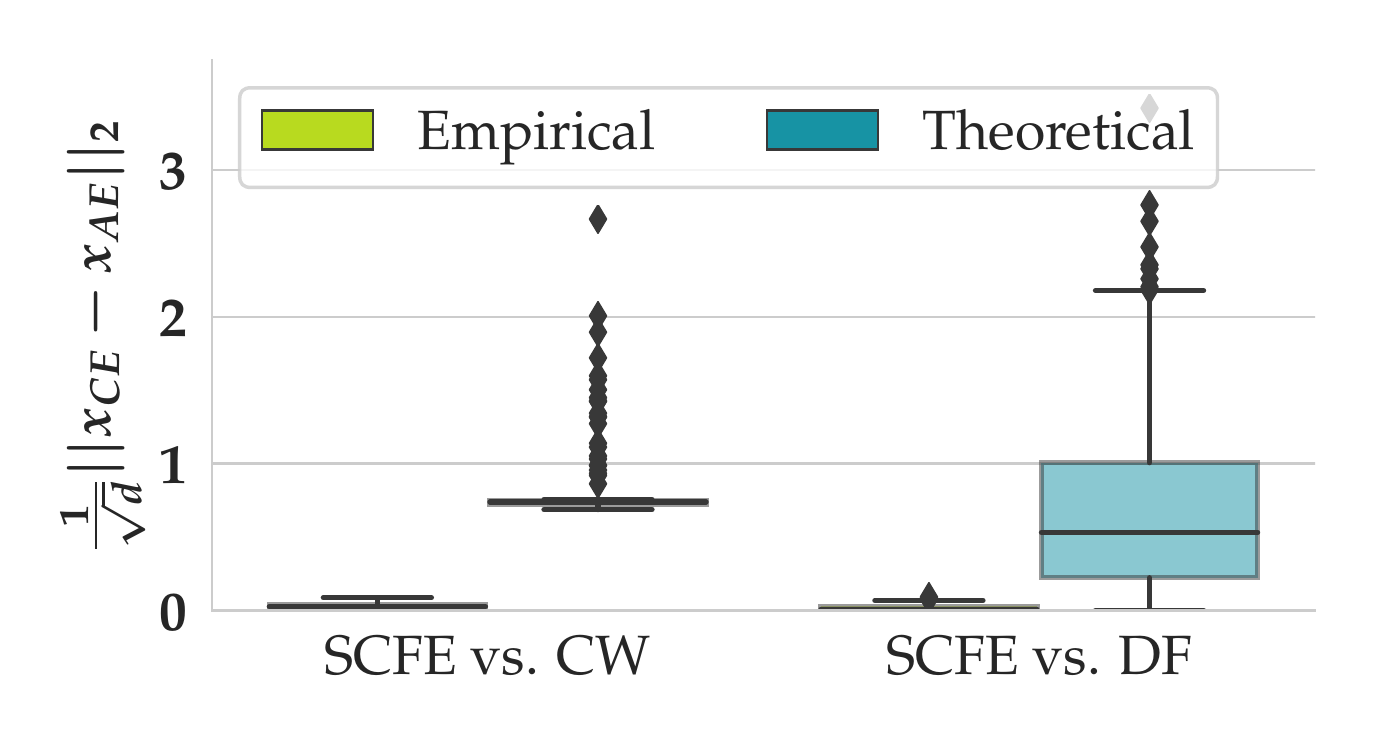}
         \caption{COMPAS -- LR}
         \label{fig:bound_compas_logistic}
\end{subfigure}
\hfill
\begin{subfigure}[b]{0.32\columnwidth}
         \centering
         \includegraphics[width=\textwidth]{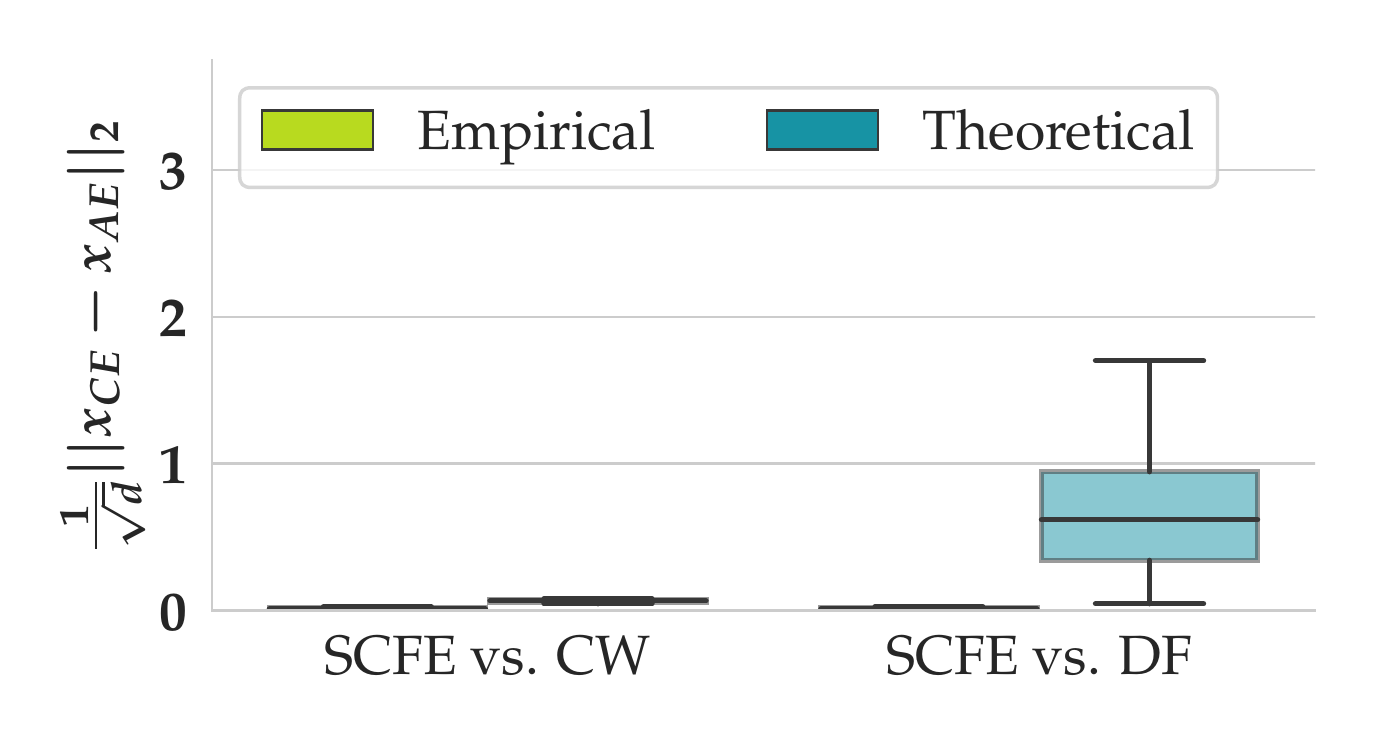}
         \caption{
         German Credit -- LR}
         \label{fig:bound_credit_logistic}
\end{subfigure}
\vfill
\begin{subfigure}[b]{0.32\textwidth}
         \centering
         \includegraphics[width=\textwidth]{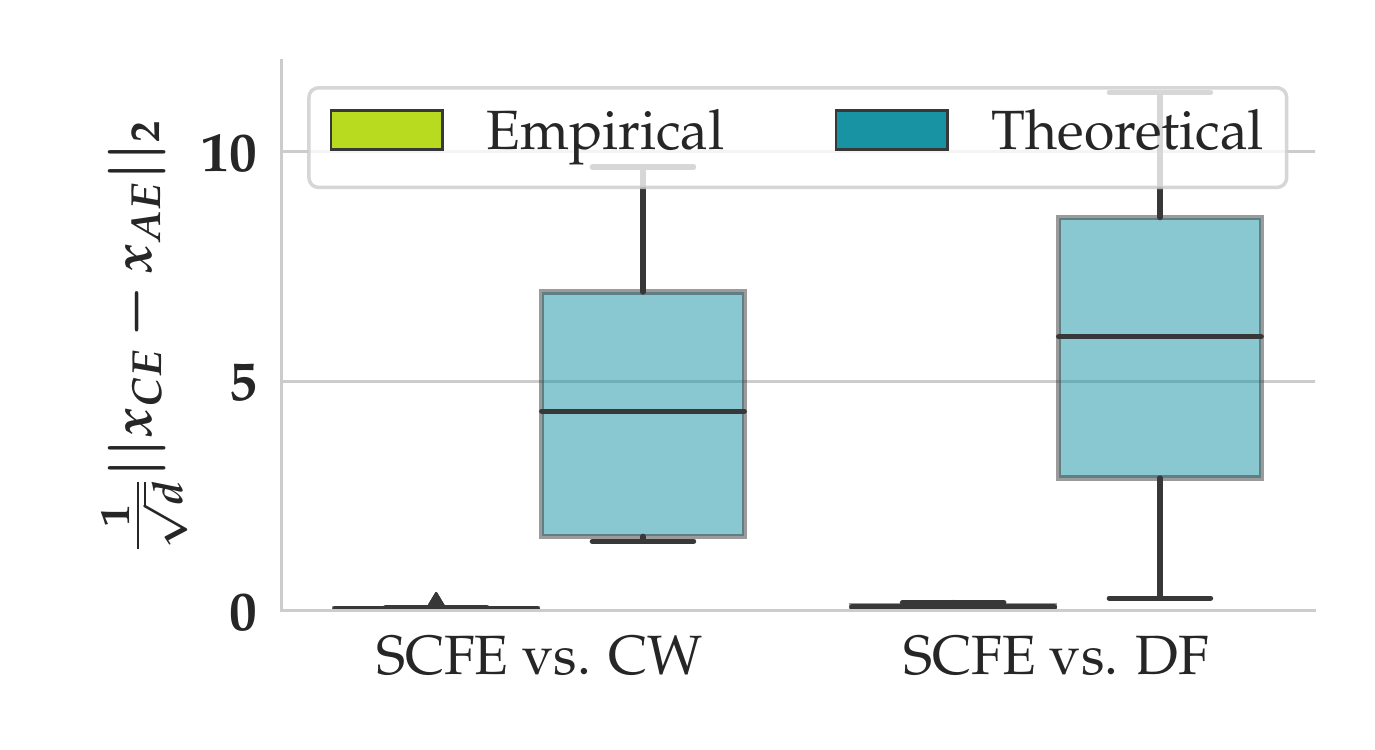}
         \caption{Adult -- ANN}
         \label{fig:bound_adult_logistic}
\end{subfigure}
~
\begin{subfigure}[b]{0.32\columnwidth}
         \centering
         \includegraphics[width=\textwidth]{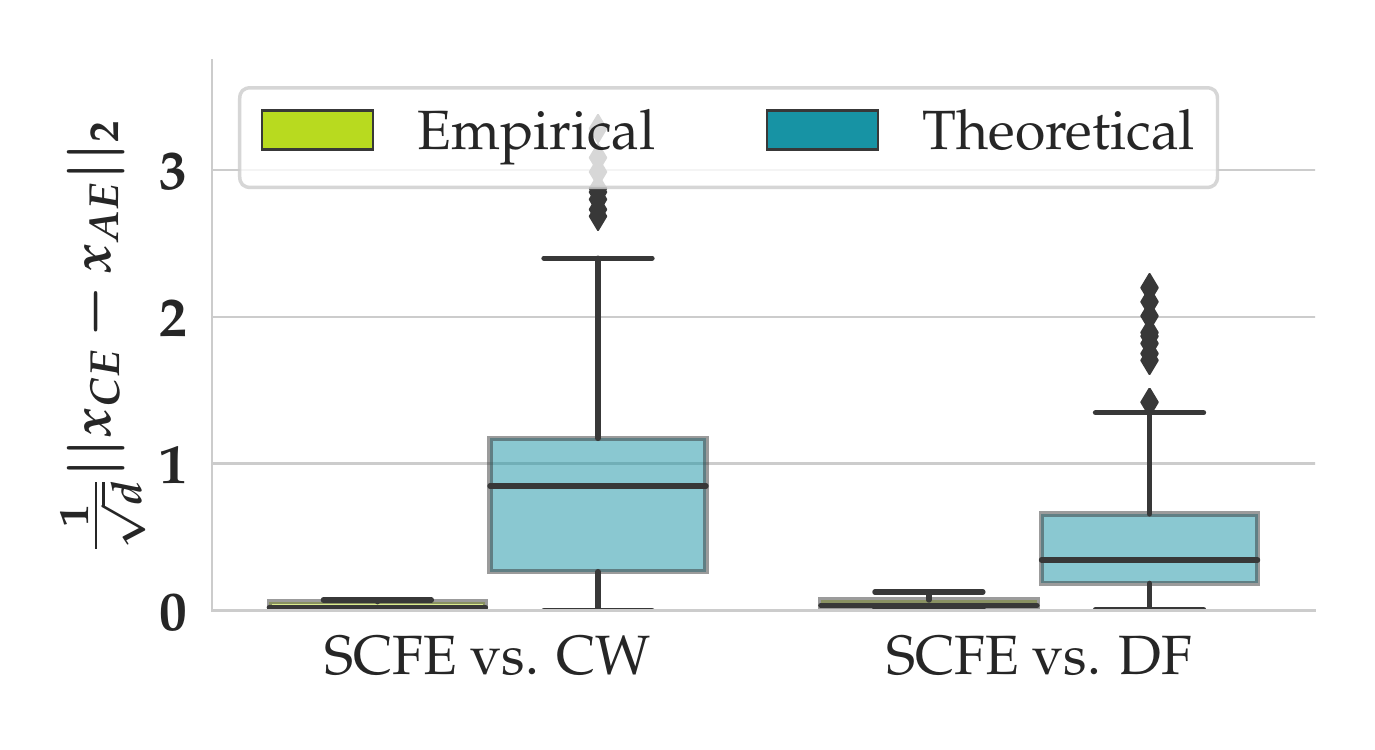}
         \caption{COMPAS -- ANN}
         \label{fig:bound_compas_logistic}
\end{subfigure}
\hfill
\begin{subfigure}[b]{0.32\columnwidth}
         \centering
         \includegraphics[width=\textwidth]{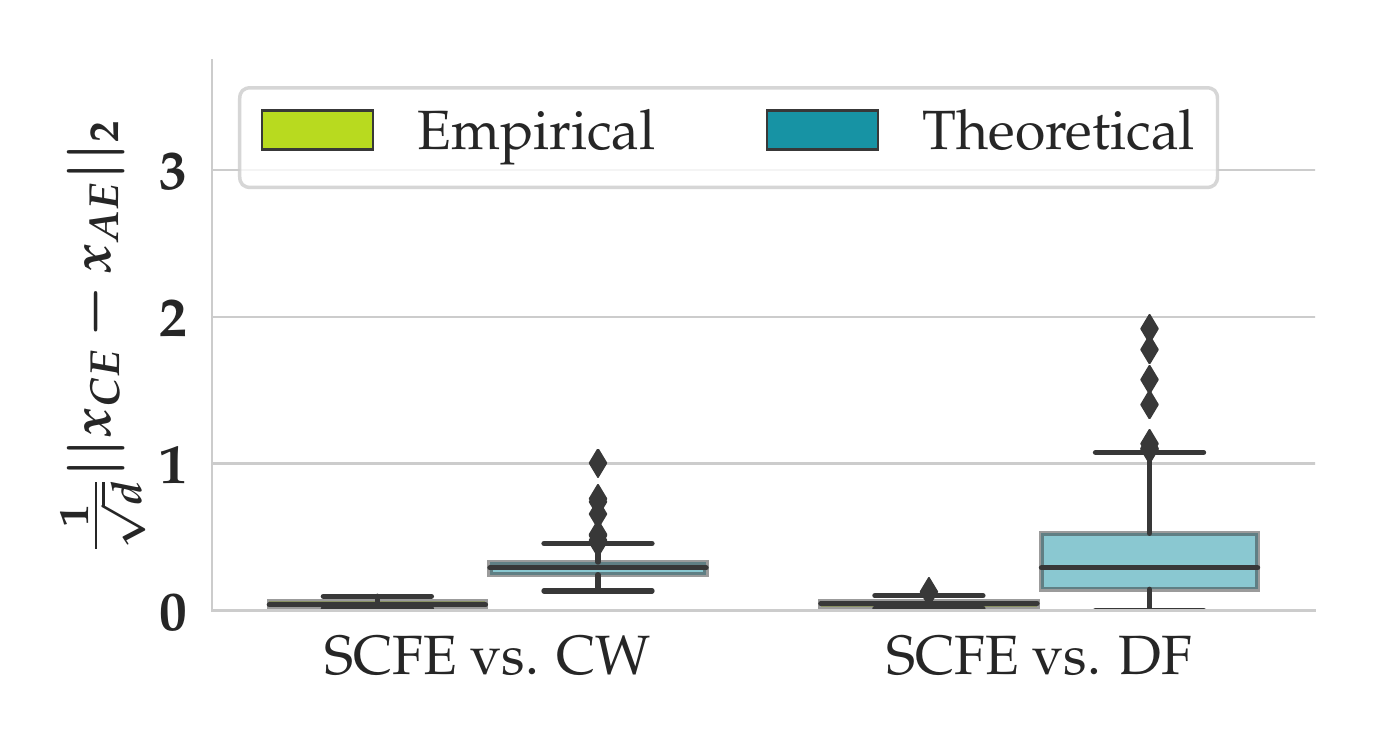}
         \caption{
         German Credit -- ANN}
         \label{fig:bound_credit_logistic}
\end{subfigure}

\caption{Verifying the theoretical bounds from Theorems \ref{thm1} and \ref{thm2}. The green boxplots correspond to the empirical norm differences 
between SCFE (i.e., $x_{\text{CE}}$) and CW or DF (i.e., $x_{\text{AE}}$). The blue boxplots show the distribution of upper bounds, which we evaluated by plugging in the necessary quantities (hyperparameters, gradients, logit values) into equations \ref{eq:w_cw_bound} and \ref{eq:w_df_bound}. No bounds are violated. For ANNs, the upper bounds were computed using local linear model approximations.}
\vspace{-0.1in}
\label{fig:boundsall}
\end{figure}
\begin{figure}[tb]
\vspace{-0.1in}
\begin{subfigure}[b]{0.32\textwidth}
         \centering
         \includegraphics[width=\textwidth]{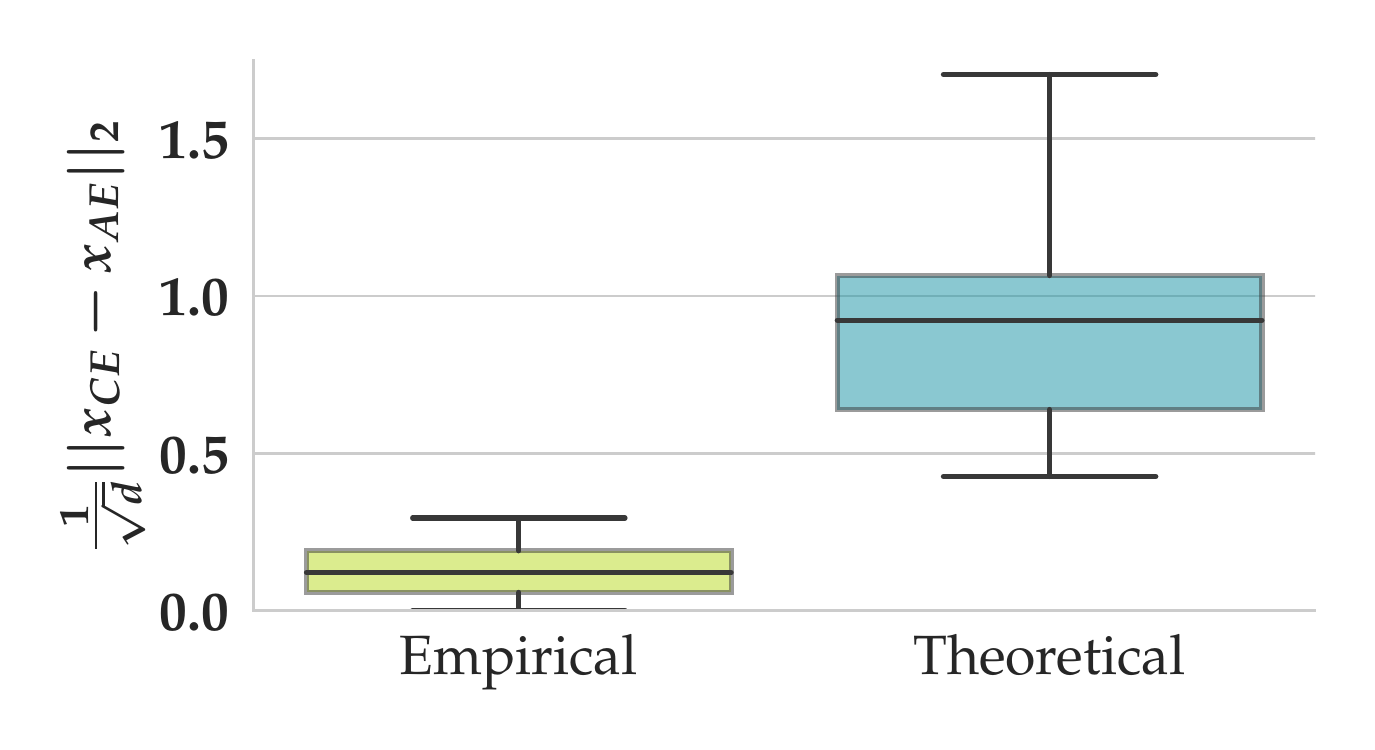}
         \caption{Adult -- ANN}
         \label{fig:bound_adult_logistic_lem2}
\end{subfigure}
\hfill
\begin{subfigure}[b]{0.32\textwidth}
         \centering
         \includegraphics[width=\textwidth]{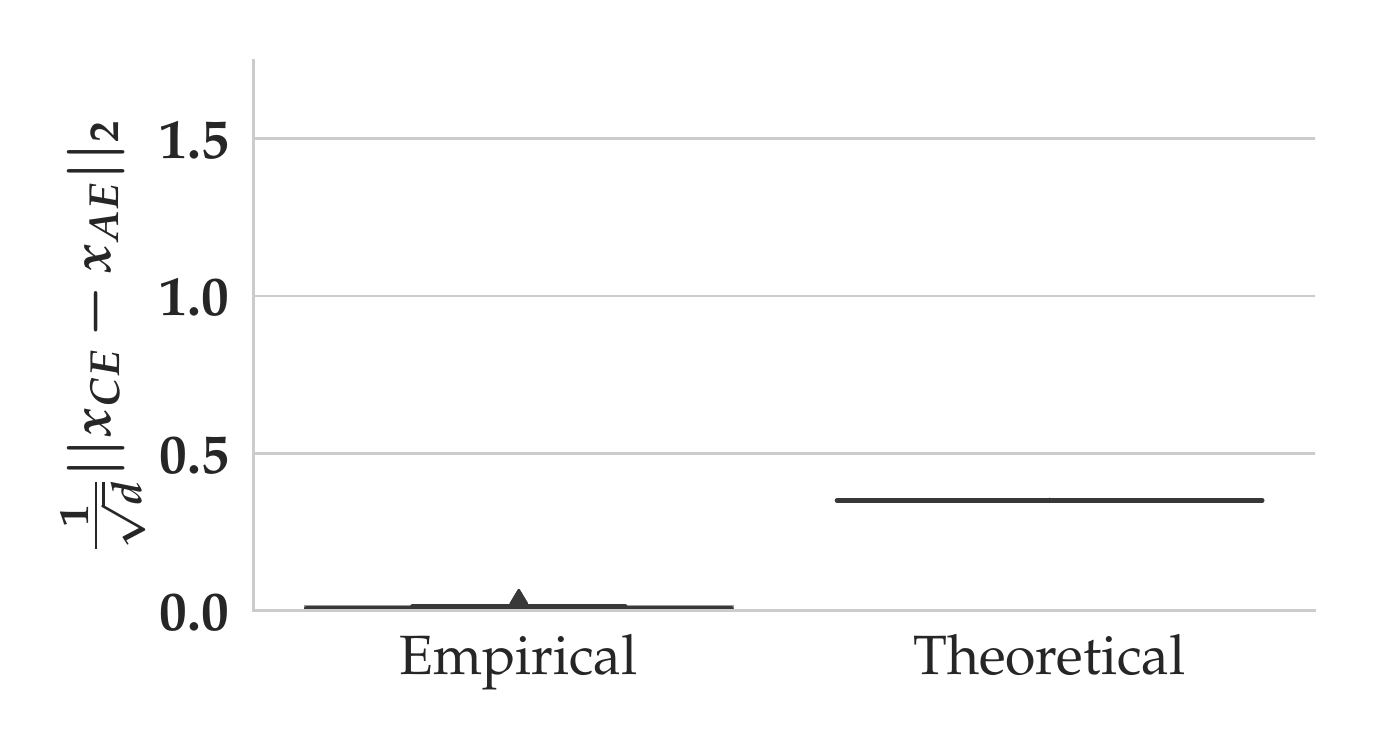}
         \caption{COMPAS -- ANN}
         \label{fig:bound_compas_logistic_lem2}
\end{subfigure}
\hfill
\begin{subfigure}[b]{0.32\textwidth}
         \centering
         \includegraphics[width=\textwidth]{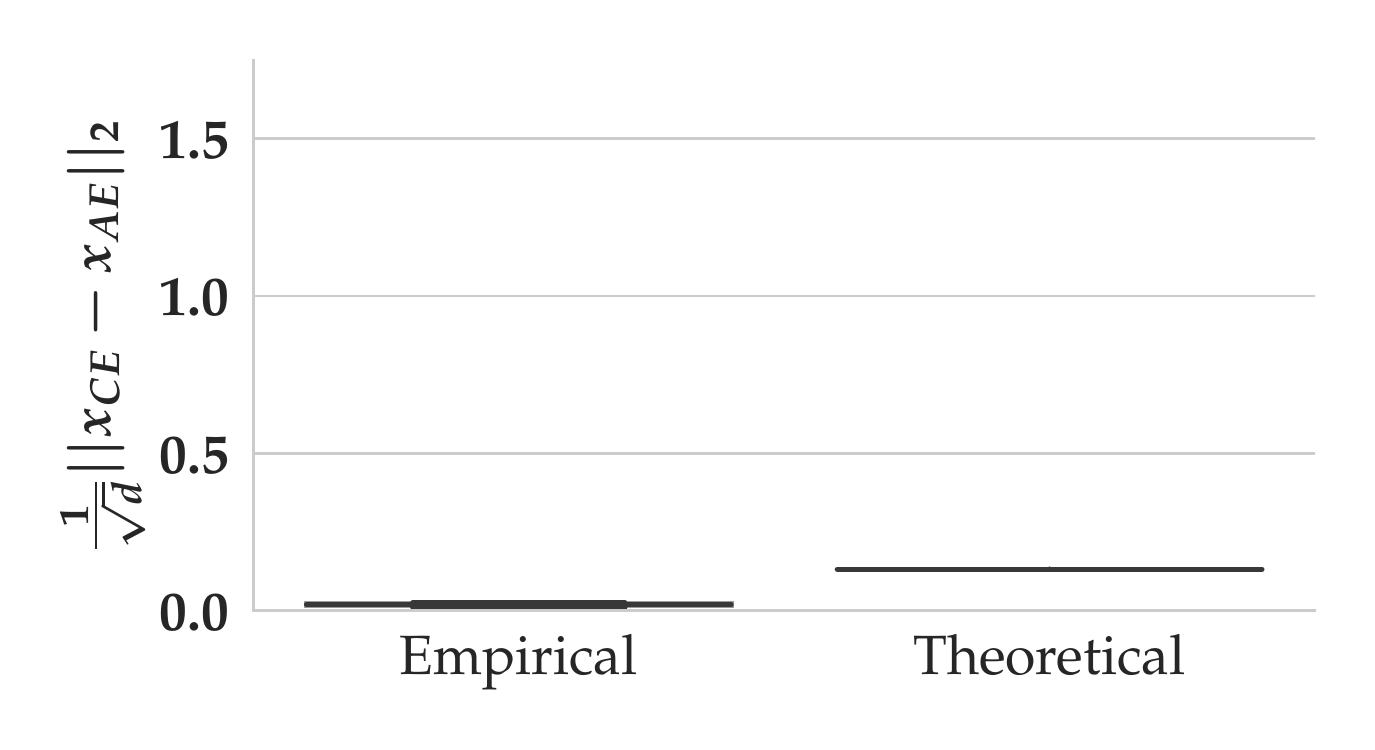}
         \caption{German Credit -- ANN}
         \label{fig:bound_credit_logistic_lem2}
\end{subfigure}
\caption{Verifying the theoretical bounds from Lemma \ref{thm3}. The green boxplots correspond to the empirical norm differences between CCHVAE (i.e., $x_{\text{CE}}$) and NAE (i.e., $x_{\text{AE}}$). The blue boxplots show the distribution of upper bounds, which we evaluated by plugging in the corresponding quantities (hyperparameters, Lipschitz constant) into the upper bound from Lemma \ref{thm3}. The Lipschitz constant is computed based on decoders and encoders using Lemma 4. No bounds are violated.}
\label{fig:bound_logistic_lem2}
\end{figure}%

\xhdr{Similarities between \Counters and \Advexs}\label{subsec:probing} Here, we qualitatively and quantitatively show the similarities between \counters and \advexs using several datasets.

\looseness=-1
\xhdr{Analysis with Synthetic Data} In Fig.~\ref{fig:synthetic_2d}, we show the similarity between \counterexps and \advexs generated for a classifier trained on a two-dimensional mixture of Gaussian datasets. Across all cases, we observe that most output samples generated by \counterexp and \advex methods overlap. In particular, for samples near the decision boundary, the solutions tend to be more similar. These results confirm our theoretical bounds, which depend on the difference between the logit sample prediction $f(\bx)$ and the target score $s$. If points are close to the decision boundary, $f(\bx)$ is closer to $s$, suggesting that the resulting \counter and \advex will be closer as implied by Theorems~\ref{thm1} and \ref{thm2}.
\begin{table*}[ht]
\centering
\caption{Average Spearman rank correlation between counterfactual and adversarial perturbations. For every input $\bx$, we compute the corresponding adversarial perturbation $\delta_{\text{AE}}$ and the counterfactual perturbation $\delta_{\text{CE}}$. We then compute Spearman's $\rho(\delta_{\text{AE}},\delta_{\text{CE}})$ and report their means (gradient-based: (g); manifold-based: (m)).}
\label{tab:rank_correlation}
\scalebox{0.81}{
\begin{tabular}{lcccccccc} 
\toprule  
& \multicolumn{4}{c}{COMPAS} & \multicolumn{4}{c}{Adult} \\
\cmidrule(lr){2-5} \cmidrule(lr){6-9}
& \multicolumn{2}{c}{LR} & \multicolumn{2}{c}{ANN} & \multicolumn{2}{c}{LR} & \multicolumn{2}{c}{ANN}  \\
\cmidrule(lr){2-3} \cmidrule(lr){4-5} \cmidrule(lr){6-7} \cmidrule(lr){8-9}
Model &      SCFE (g) &       CCHVAE (m) &  SCFE (g) &       CCHVAE (m) & SCFE (g) &       CCHVAE (m) & SCFE (g) &       CCHVAE (m) \\ 
\cmidrule(lr){1-1} \cmidrule(lr){2-2} \cmidrule(lr){3-3} \cmidrule(lr){4-4} \cmidrule(lr){5-5} \cmidrule(lr){6-6} \cmidrule(lr){7-7} \cmidrule(lr){8-8} \cmidrule(lr){9-9}
CW (g) &  $0.88 \pm   0.16$ & $0.67 \pm 0.30$ & $0.93 \pm 0.10$ & $0.67 \pm 0.22$ & $\mathbf{0.95 \pm 0.06}$ & $0.86 \pm 0.10$ & $0.92 \pm 0.09$ & $0.70 \pm 0.16$ \\ 
DF (g) &  $\mathbf{0.91 \pm  0.12}$ & $0.68 \pm 0.31$ & $\mathbf{0.97 \pm 0.03}$ & $0.65 \pm 0.22$ & $0.92 \pm 0.06$ & $0.80 \pm 0.13$ & $\mathbf{0.93 \pm 0.08}$ & $0.63 \pm 0.20$ \\  
NAE (m) &  $0.57 \pm  0.35$ & $\mathbf{0.94 \pm 0.08}$ & $0.71 \pm 0.19$ & $\mathbf{1.00 \pm 0.00}$ & $0.83 \pm 0.12$ & $\mathbf{0.90 \pm 0.10}$ &  $0.74 \pm 0.13$ & $\mathbf{0.98 \pm 0.02}$ \\ 
\bottomrule 
\end{tabular}
}
\end{table*}

\xhdr{Analysis with Real Data} For real-world datasets, we define two additional metrics beyond those studied in our theoretical analysis to gain a more granular understanding about the similarities of \counters and \advexs. First, we introduce $d_{\text{match}}$ which quantifies the similarity between \counters (i.e., $\bx_{\text{CE}}$) and \advexs (i.e., $\bx_{\text{AE}}$):
\begin{align}
    d_{\text{match}} &= \frac{1}{n} \sum_{i=1}^n \mathbbm{1} \bigg[ \frac{1}{\sqrt{d}} \lVert \bx_{\text{CE}}^{(i)} - \bx_{\text{AE}}^{(i)} \rVert_2 < \theta \bigg] \label{eq:eval_measures2},
\end{align}
\looseness=-1
where $n$ is the total number of instances used in the analysis and $\theta{\in}\{0.02, 0.05, 0.1\}$ is a threshold determining when to consider \counter and \advexs as equivalent. $d_{\text{match}}$ evaluates whether \counters and \advexs are exactly the same with higher $d_{\text{match}}$ implying higher similarity.
Second, we complement $d_{\text{match}}$ by Spearman rank $\rho$ between $\delta_{\text{CE}}$ and $\delta_{\text{AE}}$, which is a rank correlation coefficient measuring to what extent the perturbations' rankings agree, i.e., whether \advexmethods and \counterexpmethods deem the same dimensions important in order to arrive at their final outputs.
Here, $\rho(\delta_{\text{CE}},\delta_{\text{AE}}){=}1$ implies that the rankings are same, $0$ suggests that the rankings are independent, and $-1$ indicates reversely ordered rankings.
\begin{figure}[tb]
\begin{subfigure}{0.5\textwidth}
         \centering
         \includegraphics[width=\textwidth]{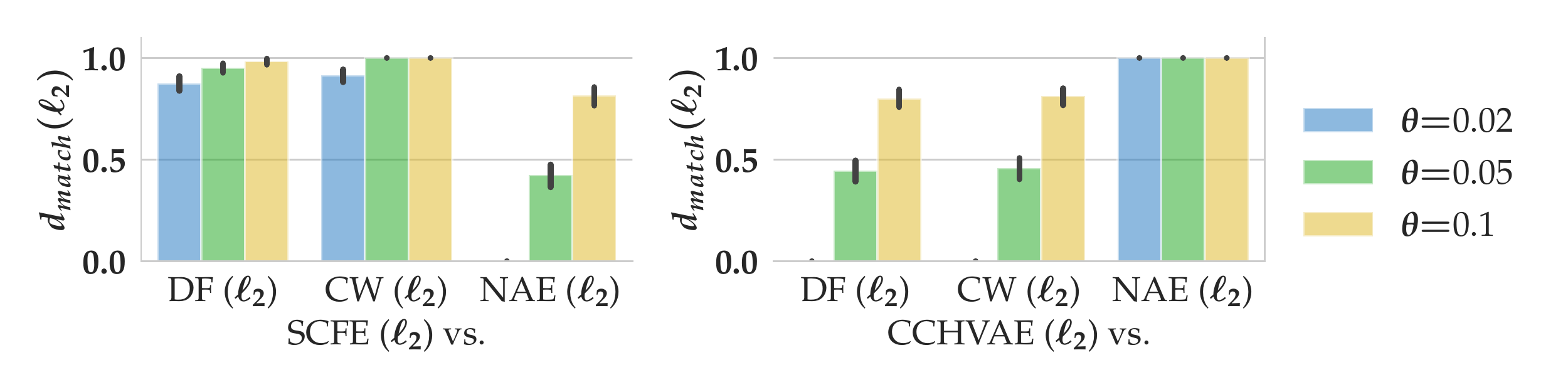}
         \vspace{-0.25in}
         \caption{COMPAS}
         \label{fig:dmatch_compas_logistic}
\end{subfigure}
\hfill
\begin{subfigure}{0.5\textwidth}
         \centering
         \includegraphics[width=\textwidth]{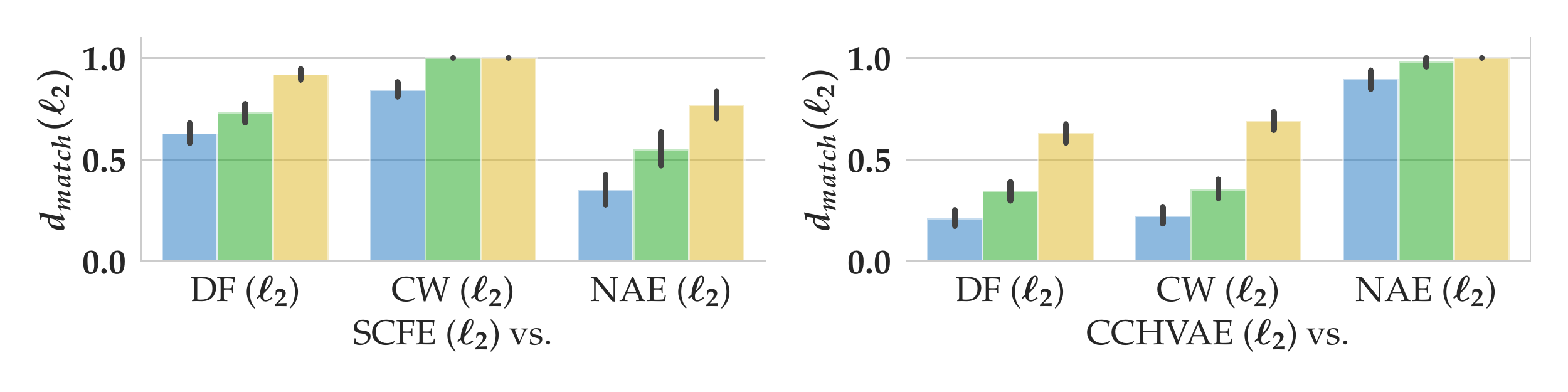}
         \vspace{-0.25in}         
         \caption{Adult}
         \label{fig:dmatch_adult_logistic}
\end{subfigure}
\vspace{-0.1in}
\caption{Analyzing to what extent different \counterexpmethods and \advexmethods are empirically equivalent for the logistic regression classifier. To do that, we compute $d_{\text{match}}$ from Eqn.~\ref{eq:eval_measures2}.
Missing bars indicate that there was no match. 
}
\vspace{-0.1in}
\label{fig:dmatch_all}
\end{figure}

In Fig.~\ref{fig:dmatch_all}, we compare a given \counterexpmethod to salient \advexmethods (\deepfool, \cw, and \zhao) using $d_{\text{match}}$. We show the results for Adult and COMPAS datasets using LR models and relegate results for German Credit as well as neural network classifiers to Appendix~\ref{app:additional_experiments}. Our results in Fig.~\ref{fig:dmatch_all} validate that the \scfe method is similar to \deepfool and \cw (higher $d_{\text{match}}$ for lower $\theta$). Across all datasets, this result aligns and validates with the similarity analysis in Sec.~\ref{sec:compare}. Similarly, manifold-based methods demonstrate higher $d_{\text{match}}$ compared to non-manifold methods (right panels in Fig.~\ref{fig:dmatch_all}). Additionally, we show the results from the rank correlation analysis in Table \ref{tab:rank_correlation} and observe that the maximum rank correlations (between 0.90 and 1.00) are obtained for methods that belong to the same categories, suggesting that the considered \counters and \advexs are close to being equivalent.

%% file: 060conclusion.tex
\section{Conclusion}
In this work, we formally analyzed the connections between state-of-the-art \advexmethods and \counterexpmethods. To this end, we first highlighted salient \counterexp and \advex methods in literature, and leveraged similarities in their objective functions, optimization algorithms and constraints utilized in these methods to theoretically analyze conditions for equivalence 
and bound the distance between the solutions output by \counterexp and \advexmethods. 
For locally linear models, we bound the distance between the solutions obtained by \cw and \scfe using loss functions preferred in the respective works. We obtained similar bounds for the solutions of \deepfool and \scfe.
We also demonstrated equivalence between the manifold-based methods of \zhao and \chvae and bounded the distance between their respective solutions.
Finally, we empirically evaluated our theoretical findings on simulated and real-world data sets. 


By establishing theoretically and empirically that several popular \counterexp algorithms are generating extremely similar solutions as those of well known \advex algorithms, our work raises fundamental questions about the design and development of existing \counterexp algorithms.
\emph{Do we really want counterfactual explanations to resemble adversarial examples, as our work suggests they do?} \emph{How can a decision maker distinguish an adversarial attack from a \counterexp?} \emph{Does this imply that decision makers are tricking their own models by issuing \counterexps?} \emph{Can we do a better job of designing \counterexps?} Moreover, by establishing connections between popular \counterexp and \advex algorithms, our work opens up the possibility of using insights from adversarial robustness literature to improve the design and development of \counterexp algorithms.

We hope our formal analysis helps carve a path for more robust approaches to \counterexps, a critical aspect for calibrating trust in ML. Improving our theoretical bounds using other strategies and deriving new theoretical bounds for other approaches is an interesting future direction. 

%% file: appendix.tex
\vspace{-5cm}
\section*{Appendix Summary} 
Section \ref{app:summary} provides a categorization of counterfactual explanation and adversarial example methods. In Section \ref{app:theory}, we provide detailed proofs for Lemmas \ref{lemma1} and \ref{lemma2}, and Theorems \ref{thm1} and \ref{thm2}. In Section \ref{app:expts}, we provide implementation details for all models used in our experiments including (i) the supervised learning models, (ii) the counterfactual explanation and adversarial example methods, and the (iii) generative models required to run the manifold-based methods. Finally, in Section \ref{app:additional_experiments}, we present the remaining experiments we referred to in the main text.

\section{Taxonomy of Counterfactual and Adversarial Example Methods}\label{app:summary}
In order to choose methods to compare across \counterexpmethods and \advexmethods, we surveyed existing literature. We use a taxonomy to categorize each subset of methods based on various factors. The main characteristics we use are based on type of method, based on widely accepted terminology and specific implementation details. In particular, we distinguish between i) constraints imposed for generating \advexs or \counterexps, ii) algorithms used for generating them. For the class of \advexmethods, we further distinguish between \emph{poisoning attacks} and \emph{evasion attacks} and note that evasion attacks are most closely related to \counterexpmethods. The taxonomy for \counterexpmethods is provided in Table~\ref{tab:rec_tax} and that for \advexmethods is provided in Table~\ref{tab:adv_tax}. 

The main algorithm types used for \counterexpmethods are search-based, gradient-based and one method that uses integer programming~\citep{Ustun_2019}. The main constraints considered are actionability i.e., only certain features are allowed to change, and \counterexps are encouraged to be realistic using either causal and/or manifold constraints. Similarly, for \advexmethods primarily, Greedy search-based and gradient-based methods are most common. Manifold constraints are also imposed in a few cases where the goal is to generate adversaries close to the data-distribution. Based on this taxonomy, we select the appropriate pairs of \counterexpmethod and \advexmethod to compare to each other for theoretical analysis. This leads us to compare gradient-based methods \scfe and \cw attack, \scfe and \deepfool and finally, manifold-based methods \chvae and \zhao with their search-based algorithms.
{
\begin{table}[h!]
\caption{Taxonomy of \counterexpmethods}
\label{tab:rec_tax}
\centering
\scalebox{1}{
\begin{tabular}{@{}lll@{}}
\toprule
Algorithm & Constraints & Method \\ \midrule
Search-based & \begin{tabular}[c]{@{}l@{}} Causal, Actionability \\ Manifold, Actionability\end{tabular} & \begin{tabular}[c]{@{}l@{}}  \textbf{MINT} \citep{karimi2020algorithmic} \\ \textbf{C-CHVAE} \citep{pawelczyk2020learning}\end{tabular}  \\ \midrule
Gradient-based & \begin{tabular}[c]{@{}l@{}} Actionability \\ Manifold, Actionability \end{tabular} & \begin{tabular}[c]{@{}l@{}} \textbf{CFE, SCFE} \cite{wachter2017counterfactual}\\ \textbf{REVISE} \citep{joshi2019towards} \end{tabular}\\ \midrule
Integer-programming & Actionability/Linear black-box & \textbf{AR} \citep{Ustun_2019} \\ \bottomrule
\end{tabular}
}
\end{table}
}
\begin{table}[h!]
\caption{Taxonomy of \advexmethods}
\label{tab:adv_tax}
\centering
\scalebox{1}{
\begin{tabular}{l@{}lll@{}}
\toprule
& Algorithm & Constraints & Method \\ \midrule
\multirow{2}{2cm}{\textbf{Poisoning Attacks}} & Greedy Search & Manifold & {\bf{Adv.\ Data Poisoning}} \citep{tavallali2021adversarial} \\
& Gradient-based & Data-domain & {\bf{SVM-attack}} \citep{biggio2012poisoning} \\
& & & {\bf{One-Shot Kill}} \citep{shafahi2018poison} \\ \midrule
\multirow{2}{2cm}{\textbf{Evasion Attacks}} & Search-based & Manifold & {\bf{NAE}}~\citep{zhao2018generating} \\
& Gradient-based & Data-domain & {\bf{DeepFool}}~\citep{moosavi2016deepfool} \\ & & & 
{\bf{C\&W Attack}}~\citep{carlini2017towards}  \\ \bottomrule
\end{tabular}
}
\end{table}

\section{Proofs for Section~\ref{sec:compare}}\label{app:theory}
\subsection{Proof of \lemref{lemma1}}\label{app:lemma1}
\begin{replemma}{lemma1}
For a linear score function $f(x) = \bw^{\top}\bx+b$, the \textbf{SCFE} counterfactual for $\bx$ on $f$ is $\bx'=\bx+\delta^{*}$ where 
\begin{equation*}
\delta^{*} =(\mathbf{\bw \bw^T} + \lambda \mathbf{I})^{-1}(s-\mathbf{w}^{T}\mathbf{x}-b)\bw.
\end{equation*}
\label{SCFE_closedform2}
\end{replemma}
\begin{proof}
 Reformulating Equation \ref{scfeobj} using $l_{2}$-norm as the distance metric, we get:
\begin{align*}
    & \min_{\mathbf{x'}}(\mathbf{w}^{T}\mathbf{x'}+b-s)^{2}+\lambda||\mathbf{x'} - \mathbf{x}||_{2}^{2}.
\end{align*}
We can convert this minimization objective into finding the minimum perturbation $\delta$ by substituting $\mathbf{x'}=\mathbf{x}+\delta$, \ie
\begin{equation}
    \min_{\delta}(\mathbf{w}^{T}\mathbf{x} + \mathbf{w}^{T}\delta + b - s)^{2}+\lambda||\mathbf{x}' - \mathbf{x}||_{2}^{2}.
\end{equation}
Using $s-\mathbf{w}^{T}\mathbf{x}-b= m$ as a dummy variable and $\mathbf{x}'-\mathbf{x}=\delta$, we get:
\begin{align*}
    & \min_{\delta}(\mathbf{w}^{T}\delta-m)^{2}+\lambda||\delta||_{2}^{2}\\
    & \min_{\delta} (\mathbf{w}^{T}\delta-m)^{T}(\mathbf{w}^{T}\delta-m) + \lambda\delta^{T}\delta\\
    & \min_{\delta}(\delta^{T}\mathbf{w}-m)(\mathbf{w}^{T}\delta-m) + \lambda\delta^{T}\delta \tag{$m$ is a scalar, hence $m^{T}=m$}\\
    & \min_{\delta}\delta^{T}\mathbf{w}\mathbf{w}^{T}\delta -2m\delta^{T}\mathbf{w} + m^{2} + \lambda\delta^{T}\delta \\
    & \min_{\delta}\delta^{T}(\mathbf{w}\mathbf{w}^{T} + \lambda\mathbf{I})\delta -2m\delta^{T}\mathbf{w} + m^{2} \\
    & \min_{\delta}\delta^{T}\mathbf{M}\delta -2m\mathbf{w}^{T}\delta + m^{2} \tag{where $\mathbf{M}=\mathbf{w}\mathbf{w}^{T} + \lambda\mathbf{I}$} \\
    & \min_{\delta}\delta^{T}\mathbf{M}\delta -2\eta^{T}\delta + m^{2} \tag{where $m\mathbf{w} = \eta$}\\
    & \min_{\delta}\delta^{T}\mathbf{M}\delta -2\eta^{T}\delta + \eta^{T}\mathbf{M}^{-1}\eta - \eta^{T}\mathbf{M}^{-1}\eta + m^{2} \\
    & \min_{\delta} (\delta - \mathbf{M}^{-1}\eta)^{T}\mathbf{M}(\delta - \mathbf{M}^{-1}\eta) - \eta^{T}\mathbf{M}^{-1}\eta + m^{2}
\end{align*}
The closed form solution is given by,
\begin{equation}
    \delta^{*}=\mathbf{M}^{-1}\eta,
    \label{eq:pert_wachter}
\end{equation}
where $M= \bw\bw^T + \lambda\mathbf{I}$.


The expression in \eqref{eq:pert_wachter} can further be simplified:
\begin{align}
    \delta^* &= m \bigg(\mathbf{I} - \frac{\bw \bw^T}{\lambda + \lVert \bw \rVert_2^2} \bigg) \bw && (\text{Sherman-Morrison Formula}) \notag \\
    &= m \bigg( \mathbf{I} \bw - \bw \frac{\lVert \bw \rVert_2^2}{\lambda + \lVert \bw \rVert_2^2} \bigg) \notag \\
    & = m \cdot \frac{\lambda}{\lambda + \lVert \bw \rVert_2^2} \cdot \bw, \label{eq:new_perturbation}
\end{align}
where $m:= s - \bw^T \bx - b$.
We can also solve for the optimal hyperparamter $\lambda$:
\begin{align*}
    s &= b + \bw^T (x + \delta^*) \tag{$\delta^*$ leads to target $s$} \\
      &= b + \bw^T x +  \bw^T \delta^* \\
    \iff m &= \bw^T \delta^* \\
   \Rightarrow m &=  m \cdot \frac{\lambda \cdot \lVert \bw \rVert_2^2}{\lambda + \lVert \bw \rVert_2^2} && \tag{\text{insert \eqref{eq:new_perturbation}}}\\
    \lambda^* &= \frac{\lVert \bw \rVert_2^2}{\lVert \bw \rVert_2^2 - 1}
\end{align*}
Notice, that the optimal hyperparameter can also be negative, if $0< \lVert \bw \rVert_2^2 < 1$!
Finally, substituting into \eqref{eq:new_perturbation} yields:
\begin{equation}
    \delta^{**} = \frac{m}{\lVert \bw \rVert_2^2} \cdot \bw.
\end{equation}
\end{proof}

\subsection{Proof of \lemref{lemma2}}\label{app:lemma2}
\begin{lemma}\label{lemma2}
For a binary classifier $h(\bx)=g(f(\bx))$ such that $f(\bx)=\bw^{\top}\bx+b$, $g(\bx)=\sigma(\bx)$, and $h(\bx)$ is the probability that $\bx$ is in the class $y=1$,
\[\ell^{*}(\bx) = \max(0,-2(\bw^{\top}\bx+b))\].
\end{lemma}
\begin{proof}
Given our formulation of $h(\bx)$, $f(\bx)$ is the score corresponding to class $y=1$. By the definition of $\sigma(\bx)$, \[f(\bx)=\ln{\frac{h(\bx)}{1-h(\bx)}}=\ln{h(\bx)}-\ln{(1-h(\bx))}\]
Then the score corresponding to the class $y=0$ is 
\[\ln{\frac{1-h(\bx)}{1-(1-h(\bx))}}=\ln{\frac{1-h(\bx)}{h(\bx)}}=\ln{(1-h(\bx))}-\ln{h(\bx)}=-f(\bx)\]
Substituting back into definition of $\ell^{*}(\bx)$,
\begin{align*}
    \ell^{*}(\bx) &= \max(0,\max_{i}(f(\bx)_{i})-f(\bx)_{y})\\
    &= \max(0,(-f(\bx)-f(\bx))\\
    &= \max(0,(-2f(\bx))\\
    &= \max(0,-2(\bw^{\top}\bx+b)).
\end{align*}
\end{proof}

\subsection{Proof of \thref{thm1}}\label{app:thm1}

\begin{reptheorem}{thm1}
For a linear classifier $h(\bx) = g(f(\bx))$ such that $f(\bx) = \bw^T\bx + b$, the difference between the \textbf{SCFE} counterfactual example $\mathbf{x}_{\text{SCFE}}$ and the \textbf{C\&W} adversarial example $\mathbf{x}_{\text{CW}}$ using the recommended loss function $\ell^{*}(\cdot)= \max(0,\max_{i}(f(\bx)_{i})-f(\bx)_{y})$ is given by:
\[\lVert\mathbf{x}_{\text{SCFE}} - \mathbf{x}_{\text{CW}}\rVert_{p} \leq \biggl\lVert\frac{1}{\lambda} \bigg(\mathbf{I} - \frac{\mathbf{w}\mathbf{w}^{T}}{\lambda+\mathbf{w}^{T}\mathbf{w}}\bigg)(s-f(\mathbf{x})) - c\mathbf{I} \biggr\rVert_{p}~\lVert\mathbf{w}\rVert_{p}\].
\end{reptheorem}
\begin{proof}
Consider a binary classifier $h(\bx)=g(f(\bx))$ such that $f(\bx) =\bw^{\top}\bx + b$, $g(\bx)=\sigma(\bx)$, and $h(\bx)$ is the probability that $\bx$ is in the class $y=1$. Then by \lemref{lemma2} and using $\ell_2$-nrom as the distance metric, we can write the \textbf{C\&W Attack} objective as
\[\arg\min_{\bx'} c\max(0,-2(\bw^{\top}\bx'+b))+\lVert\bx- \bx'\rVert_{2}^{2}\]
We can convert this minimization objective into finding the minimum perturbation $\delta$ by substituting $\mathbf{x'}=\mathbf{x}+\delta$,
\[\arg\min_{\delta} c\max(0,-2(\bw^{\top}\bx+\bw^{\top}\delta+b))+\lVert\delta\rVert_{2}^{2}\]
The subgradients of this objective are
\[ \begin{cases} 
      2\delta & \mbox{when } -2(\bw^{\top}\bx+\bw^{\top}\delta+b)<0 \\
      -2c\bw + 2\delta & \mbox{otherwise}
   \end{cases}
\]
By \lemref{lemma2},
$-2(\bw^{\top}\bx+\bw^{\top}\delta+b)=-f(\bx)-f(\bx)<0$. This implies that $f(\bx)>-f(\bx)$, i.e that the score for class $y=1$ is greater than the score for $y=0$. As this indicates an adversarial example has already been found, we focus on minimizing the other subgradient. Setting this subgradient equal to 0,
\begin{align*}
    0 &= -2c\bw + 2\delta\\
    \delta &= c\bw 
\end{align*}
Thus the minimum perturbation to generate and adversarial example using the \textbf{C\&W Attack} is 
\[\delta_{\text{CW}}^{*} = c\bw\] 
Now, taking the difference between the minimum perturbation to generate a \textbf{SCFE} counterfactual (\lemref{lemma1}) and DeepFool (\eqref{eq:pert_df}), we get:
\begin{align*} 
    \delta^{*}_{\text{SCFE}} - \delta^{*}_{\text{CW}} &= (\mathbf{w}\mathbf{w}^{T} + \lambda\mathbf{I})^{-1}(s-\mathbf{w}^{T}\mathbf{x}-b)\mathbf{w} - c\bw\\
    &= \big((\mathbf{w}\mathbf{w}^{T} + \lambda\mathbf{I})^{-1}(s-f(\bx)) - c\mathbf{I}\big)\bw \\
     &= \bigg(\frac{1}{\lambda}\bigg(\mathbf{I} - \frac{\mathbf{w}\mathbf{w}^{T}}{\lambda+\mathbf{w}^{T}\mathbf{w}}\bigg)(s-f(\mathbf{x})) - c\mathbf{I}\bigg) \mathbf{w} \tag{Using Sherman–Morrison formula}
\end{align*}
Taking the $l_{p}$-norm on both sides, we get:
\begin{align*}
    \lVert\delta^{*}_{\text{SCFE}} - \delta^{*}_{\text{CW}}\rVert_{p} &= \biggl\lVert\bigg(\frac{1}{\lambda}\bigg(\mathbf{I} - \frac{\mathbf{w}\mathbf{w}^{T}}{\lambda+\mathbf{w}^{T}\mathbf{w}}\bigg)(s-f(\mathbf{x})) - c\mathbf{I}\bigg) \mathbf{w}\biggr\rVert_{p} \\
    & \leq \biggl \lVert \frac{1}{\lambda}\bigg(\mathbf{I} - \frac{\mathbf{w}\mathbf{w}^{T}}{\lambda+\mathbf{w}^{T}\mathbf{w}}\bigg)(s-f(\mathbf{x})) - c\mathbf{I} \biggr\rVert_{p}~\lVert\mathbf{w}\rVert_{p} \tag{Using Cauchy-Schwartz}
\end{align*}

Adding and subtracting the input instance $\mathbf{x}$ in the left term, we get:
\begin{align*}
    \Vert\mathbf{x} + \delta^{*}_{\text{SCFE}} - (\mathbf{x} + \delta^{*}_{\text{CW}})\Vert_{p} &\leq 
    \biggl \lVert \frac{1}{\lambda}\bigg(\mathbf{I} - \frac{\mathbf{w}\mathbf{w}^{T}}{\lambda+\mathbf{w}^{T}\mathbf{w}}\bigg)(s-f(\mathbf{x})) - c\mathbf{I} \biggr\Vert_{p}~ \lVert\mathbf{w} \rVert_{p} \\
    \lVert\mathbf{x}_{\text{SCFE}} - \mathbf{x}_{\text{CW}}\rVert_{p} &\leq \biggl\lVert\frac{1}{\lambda}\bigg(\mathbf{I} - \frac{\mathbf{w}\mathbf{w}^{T}}{\lambda+\mathbf{w}^{T}\mathbf{w}}\bigg)(s-f(\mathbf{x})) - c\mathbf{I} \biggr\Vert_{p}~\Vert\mathbf{w}\Vert_{p},
\end{align*}
where the final equation gives an upper bound on the difference between the \textbf{SCFE} counterfactual and the \textbf{C\&W} adversarial example.

Furthermore, we ask under which conditions the normed difference becomes 0. We start with:
\begin{align*} 
    \delta^{*}_{\text{SCFE}} - \delta^{*}_{\text{CW}} &= m \cdot \frac{\lambda}{\lambda + \lVert \bw \rVert_2^2} \cdot \bw - c \cdot \bw
\end{align*}
Taking the $l_{p}$-norm on both sides, we get:
\begin{align*} 
    \lVert \delta^{*}_{\text{SCFE}} - \delta^{*}_{\text{CW}} \rVert_p &= \biggl| m \cdot \frac{\lambda}{\lambda + \lVert \bw \rVert_2^2} - c \biggr| \cdot \lVert \bw \Vert_p
\end{align*}
If we were to use the optimal hyperparameters we would get:
\begin{align*}
\lVert \delta^{**}_{\text{SCFE}} - \delta^{*}_{\text{CW}} \rVert_p =  \biggl| \frac{m - c \cdot \lVert \bw \rVert_2^2}{\lVert \bw \rVert_2^2} \biggr| \cdot \lVert \bw \rVert_p,
\end{align*}
where equality holds when the hyperparameter is chosen so that $c:= \frac{m}{ \lVert \bw \rVert_2^2}.$
\end{proof}

\subsection{Proof of~\thref{thm2}}\label{app:thm2}

\begin{reptheorem}{thm2}
For a linear classifier $h(\bx) = g(f(\bx))$ such that $f(\bx) = \bw^T\bx + b$, the difference between the counterfactual example $\mathbf{x}_{\text{SCFE}}$ generated by~\citet{wachter2017counterfactual} and the adversarial example $\mathbf{x}_{\text{DF}}$ generated by \citet{moosavi2016deepfool} is given by:
\begin{equation}
    ||\mathbf{x}_{\text{SCFE}} - \mathbf{x}_{\text{DF}}||_{p} \leq \biggl \lVert \frac{1}{\lambda} \bigg(\mathbf{I} - \frac{\mathbf{w}\mathbf{w}^{T}}{\lambda+\mathbf{w}^{T}\mathbf{w}} \bigg)(s-f(\mathbf{x})) + \bigg(\bI \frac{f(\mathbf{x})}{\lVert \mathbf{w} \rVert_{2}^{2}}\bigg) \biggr \rVert_{p} ~ \cdot ||\mathbf{w}||_{p},
    \label{eq:w_df_bound2}
\end{equation}
\end{reptheorem}

\begin{proof}

The minimal perturbation to change the classifier’s decision for a binary model $f(\mathbf{x})$ is given by the closed-form formula \citep{moosavi2016deepfool}:
\begin{equation}
    \delta^{*}_{\text{DF}} = -\frac{f(\mathbf{x})}{||\mathbf{w}||_{2}^{2}} \mathbf{w}
    \label{eq:pert_df}.
\end{equation}

Now, taking the difference between the minimum perturbation added to an input instance $\mathbf{x}$ by Wachter algorithm (\lemref{lemma1}) and DeepFool (\eqref{eq:pert_df}), we get:
\begin{align*} 
    & \delta^{*}_{\text{SCFE}} - \delta^{*}_{\text{DF}} = (\mathbf{w}\mathbf{w}^{T} + \lambda\mathbf{I})^{-1}(s-\mathbf{w}^{T}\mathbf{x}-b)\mathbf{w} - \bigg( -\frac{f(\mathbf{x})}{||\mathbf{w}||_{2}^{2}} \mathbf{w} \bigg)\\
    & \delta^{*}_{\text{SCFE}} - \delta^{*}_{\text{DF}} = \bigg((\mathbf{w}\mathbf{w}^{T} + \lambda\mathbf{I})^{-1}(s-f(\mathbf{x})) + \frac{f(\mathbf{x})}{||\mathbf{w}||_{2}^{2}}\bigg) \mathbf{w}\\
    & \delta^{*}_{\text{SCFE}} - \delta^{*}_{\text{DF}} = \bigg(\frac{1}{\lambda}\bigg(\mathbf{I} - \frac{\mathbf{w}\mathbf{w}^{T}}{\lambda+\mathbf{w}^{T}\mathbf{w}}\bigg)(s-f(\mathbf{x})) + \frac{f(\mathbf{x})}{||\mathbf{w}||_{2}^{2}}\bigg) \mathbf{w} \tag{Using Sherman–Morrison formula}
\end{align*}
Taking the $l_{p}$-norm on both sides, we get:
\begin{align*}
    \lVert\delta^{*}_{\text{SCFE}} - \delta^{*}_{\text{DF}}\Vert_{p} &= \biggl\Vert\bigg(\frac{1}{\lambda}\bigg(\mathbf{I} - \frac{\mathbf{w}\mathbf{w}^{T}}{\lambda+\mathbf{w}^{T}\mathbf{w}}\bigg)(s-f(\mathbf{x})) + \bI \frac{f(\mathbf{x})}{\Vert\mathbf{w}\Vert_{2}^{2}}\bigg) \mathbf{w} \biggr \Vert_{p}\\
    & \leq \biggl\Vert\frac{1}{\lambda}\bigg(\mathbf{I} - \frac{\mathbf{w}\mathbf{w}^{T}}{\lambda+\mathbf{w}^{T}\mathbf{w}}\bigg)(s-f(\mathbf{x})) + \bI \frac{f(\mathbf{x})}{\Vert\mathbf{w}\Vert_{2}^{2}} \biggr\Vert_{p}~\Vert\mathbf{w}\Vert_{p} \tag{Using Cauchy-Schwartz}
\end{align*}

Adding and subtracting the input instance $\mathbf{x}$ in the left term, we get:
\begin{align*}
    \Vert\mathbf{x} + \delta^{*}_{\text{SCFE}} - (\mathbf{x} + \delta^{*}_{\text{DF}})\Vert_{p} & \leq \biggl\Vert\frac{1}{\lambda}\bigg(\mathbf{I} - \frac{\mathbf{w}\mathbf{w}^{T}}{\lambda+\mathbf{w}^{T}\mathbf{w}}\bigg)(s-f(\mathbf{x})) + \bI \frac{f(\mathbf{x})}{\Vert\mathbf{w}\Vert_{2}^{2}} \biggr\Vert_{p}~\Vert\mathbf{w}\Vert_{p}\\
     \Vert\mathbf{x}_{\text{SCFE}} - \mathbf{x}_{\text{DF}}\Vert_{p} &\leq \biggl\Vert\frac{1}{\lambda}\bigg(\mathbf{I} - \frac{\mathbf{w}\mathbf{w}^{T}}{\lambda+\mathbf{w}^{T}\mathbf{w}}\bigg)(s-f(\mathbf{x})) + \bI \frac{f(\mathbf{x})}{\Vert\mathbf{w}\Vert_{2}^{2}} \biggr\Vert_{p}~\Vert\mathbf{w}\Vert_{p},
\end{align*}
where the final equation gives an upper bound on the difference between the \textbf{SCFE} counterfactual and the adversarial example from DeepFool \citep{moosavi2016deepfool}.

Furthermore, we ask under which conditions the normed difference becomes 0.
\begin{align*} 
    \delta^{*}_{\text{SCFE}} - \delta^{*}_{\text{DF}} &= m \cdot \frac{\lambda}{\lambda + \lVert \bw \rVert_2^2} \cdot \bw + \frac{f(\mathbf{x})}{\Vert\mathbf{w}\Vert_{2}^{2}} \mathbf{w} \\
    & = \frac{m\cdot \lambda \cdot \lVert \bw \rVert_{2}^{2} + f(\bx)  \cdot (\lambda + \lVert \bw \rVert_{2}^{2})}{(\lambda + \lVert \bw \rVert_{2}^{2}) \cdot \lVert \bw \rVert_{2}^{2}} \cdot \bw \\
    & = \frac{f(\bx) \cdot \lVert \bw \rVert_2^2 \cdot (1-\lambda) + \lambda \cdot (f(\bx) + s \cdot \lVert \bw \rVert_2^2)}{(\lambda + \lVert \bw \rVert_{2}^{2}) \cdot \lVert \bw \rVert_{2}^{2}} \cdot \bw
\end{align*}
Taking the $l_{p}$-norm on both sides and , we get:
\begin{align*} 
    \lVert \delta^{*}_{\text{SCFE}} - \delta^{*}_{\text{DF}} \rVert_p &= \biggl|
    \frac{f(\bx) \cdot \lVert \bw \rVert_2^2 \cdot (1-\lambda) + \lambda \cdot (f(\bx) + s \cdot \lVert \bw \rVert_2^2)}{(\lambda + \lVert \bw \rVert_{2}^{2}) \cdot \lVert \bw \rVert_{2}^{2}} \biggr| \cdot \Vert\mathbf{w}\Vert_{p}
\end{align*}
If we were to use the optimal hyperparameter $\lambda^*$ we would get:
\begin{align*}
\lVert \delta^{**}_{\text{SCFE}} - \delta^{*}_{\text{DF}} \rVert_p &= \biggl\lVert \frac{-f(\bx) + s}{\lVert \bw \rVert_2^2} \cdot \bw + \frac{-f(\bx)}{\lVert \bw \rVert_2^2} \cdot \bw \biggr \rVert_p 
\\ &=\frac{|s|}{\lVert \bw \rVert_2^2}  \cdot \lVert \bw \rVert_p,
\end{align*}
where equality holds when the target score is chosen so that $s{=}0$, which corresponds to a probability of $Y{=}1$ of $0.5$.

\end{proof}

\subsection{Proof of~\lemref{thm3}}
\label{app:thm3} 

\begin{replemma}{thm3}
Let $\tilde{\bz}_{NAE}$ be the solution returned~\citet[Algorithm 1]{zhao2018generating} and $\tilde{\bz}_C$ the solution returned by the counterfactual search algorithm of~\citet{pawelczyk2020learning} by sampling from $\ell_p$-norm ball in the latent space using an $L$-Lipschitz generative model $\cG_{\theta}(\cdot)$. Analogously, let $\bx_{NAE} = \cG_{\theta}(\tilde{\bz}_{\text{NAE}})$ and $\bx_C = \cG_{\theta}(\tilde{\bz}_C)$ by design of the two algorithms. Let $r_{\text{NAE}}^*$ and $r_C^*$ be the corresponding radius chosed by each algorithm respectively that successfully returns an \advex or \counterexp.  Then, $\|\bx_{\text{NAE}} - \bx_{C}\| \leq L(r_{\text{NAE}}^* + r_C^*)$.
\end{replemma}
\begin{proof}
The proof straightforwardly follows from triangle inequality and $L$-Lipschitzness of the Generative model:
\begin{align}
        \|\bx_{\text{NAE}} - \bx_C\|&= \|\cG_{\theta}(\tilde{\bz}_{\text{NAE}})- \cG_{\theta}(\tilde{\bz}_C)\|_p \\
        &\leq \|\cG_{\theta}(\tilde{\bz}_{\text{NAE}})- \bx + \bx - \cG_{\theta}(\tilde{\bz}_C)\|_p \label{eq:triangleq} \\
        &\leq \|\cG_{\theta}(\tilde{\bz}_{\text{NAE}})- \bx \|_p + \|\bx - \cG_{\theta}(\tilde{\bz}_C)\|_p \\
        &= \|\cG_{\theta}(\tilde{\bz}_{\text{NAE}})- \cG_{\theta}(\bz) \|_p + \|\cG_{\theta}(\bz) - \cG_{\theta}(\tilde{\bz}_C)\|_p \\
        &\leq L\|\tilde{\bz}_{\text{NAE}}- \bz \|_p + L\|\bz - \tilde{\bz}_C\|_p \label{eq:lipschitz}\\ 
        &\leq L \{r_{\text{NAE}}^* + r_C^*\} \label{eq:algo}
\end{align}
where \eqref{eq:triangleq} follows from triangle inequality in the $\ell_p$-norm, \eqref{eq:lipschitz} follows from the Lipschitzness assumption and \eqref{eq:algo} follows from properties of the counterfactual search algorithms.
\end{proof}

In the following we outline a lemma that allows us to estimate the Lipschitz constant of the generative model. This will be used for empirical validation of our theoretical claims.
\begin{lemma}[\citet{bora2017compressed}]
If $G$ is a $d$-layer neural network with at most $c$ nodes per layer, all weights $\leq w_{\max }$ in absolute value, and $M$ -Lipschitz non-linearity after each layer, then $G(\cdot)$ is $L$ -Lipschitz with $L=\left(M c w_{\max }\right)^{d}$.
\label{lemma:lipschitz}
\end{lemma}

\section{Experimental Setup}\label{app:expts}
\subsection{Implementation Details for Counterfactual Explanation and Adversarial Example Methods}
For all datasets, categorical features are one-hot encoded and data is scaled to lie between $0$ and $1$. We partition the dataset into train-test splits. The training set is used to train the classification models for which \advexs and \counterexps are generated. \advexs and \counterexps are generated for all samples in the test split for the fixed classification model. For \counterexp methods applied to generate recourse examples, all features are assumed actionable for comparison with \advexs methods. \Advexs and \counters are appropriately generated using the prescribed algorithm implementations in each respective method. Specifically,
\begin{asparaenum}
\item[i)] {\bf{SCFE}}: As suggested in~\citet{wachter2017counterfactual}, an Adam optimizer~\citep{kingma2014adam} is used to obtain \counterexps  corresponding to the cost function of \eqref{eq:pert_wachter}. 
\item[ii)] {\bf{C-CHVAE}}: A (V)AE is additionally trained to model the data-manifold as prescribed in~\citet{pawelczyk2020learning}. As suggested in~\citet{pawelczyk2020learning}, a counterfactual search algorithm in the latent space of the (V)AEs. Particularly, a latent sample within an $\ell_p$-norm ball with a fixed search radius is used until a counterfactual example is successfully obtained. The search radius of the norm ball is increased until a \counterexp is found. The architecture of the generative model is provided in Appendix~\ref{app:vae_architecture}. 
\item[iv)] {\bf{C\&W Attack}}: As prescribed in~\citet{carlini2017towards}, we use gradient-based optimization to find Adversarial Examples using this attack. 
\item[v)] {\bf{DeepFool}}: We implement~\citet[Algorithm 1]{moosavi2016deepfool} to generate Adversarial Examples using DeepFool.
\item[vi)] {\bf{\zhao}}: This method trains a generative model and an inverter to generate Adversarial Examples. For consistency of comparison with {\bf{C-CHVAE}}, we use the decoder of the same (V)AE as the generative model for this method. The inverter then corresponds to the encoder of the (V)AE. We use~\citet[Algorithm 1]{zhao2018generating} which uses an iterative search method to find natural adversarial examples. The algorithm searches for \advexs in the latent space with radius between $(r, r+\Delta r]$. The search radius is iteratively increased until an Adversarial Example is successfully found.
\end{asparaenum} 
We describe architecture and training details for real-world data sets in the following.

\subsection{Supervised Classification Models}

All models are implemented in PyTorch and use a $80-20$ train-test split for model training and evaluation. We evaluate model quality based on the model accuracy. All models are trained with the same architectures across the data sets:

\begin{table}[H]
\centering
\begin{tabular}{@{}lll@{}}
\toprule
                         & Neural Network                       & Logistic Regression        \\ \midrule
Units                    & [Input dim.\ , 18, 9, 3, 1] & [Input dim.\ , 1] \\
Type                     & Fully connected                      & Fully connected            \\
Intermediate activations & ReLU                                & N/A                        \\
Last layer activations   & Sigmoid                              & Sigmoid                    \\ \bottomrule
\end{tabular}
\caption{Classification model details}
\label{tab:model_dets}
\end{table}

  


\begin{table}[H]
\centering
\begin{tabular}{@{}l|llll@{}}
\toprule
              &                                                               & Adult & COMPAS & German Credit \\ \midrule
Batch-size    & NN                                                            & 512   & 32     & 64            \\ \cmidrule(l){2-5} 
              & \begin{tabular}[c]{@{}l@{}}Logistic\\ Regression\end{tabular} & 512   & 32     & 64            \\ \midrule
Epochs        & NN                                                            & 50    & 40     & 30            \\ \cmidrule(l){2-5} 
              & \begin{tabular}[c]{@{}l@{}}Logistic\\ Regression\end{tabular} & 50    & 40     & 30            \\ \midrule
Learning rate & NN                                                            & 0.002 & 0.002  & 0.001         \\ \cmidrule(l){2-5} 
              & \begin{tabular}[c]{@{}l@{}}Logistic\\ Regression\end{tabular} & 0.002 & 0.002  & 0.001         \\ \bottomrule
\end{tabular}
\caption{Training details}
\label{tab:train_dets}
\end{table}

\begin{table}[H]
\centering
\begin{tabular}{@{}llll@{}}
\toprule
 & Adult & COMPAS & German Credit \\ \midrule
Logistic Regression & 0.83  & 0.84  & 0.71  \\
Neural Network & 0.84  & 0.85  & 0.72  \\ \bottomrule
\end{tabular}
\caption{Performance of models used for generating \advexs and \counterexps}
\label{tab:classifier_acc}
\end{table}

\subsection{Generative model architectures used for \chvae and \zhao}\label{app:vae_architecture}
For the results in Lemma \ref{lemma2}, we used linear encoders and decoders. For the remaining experiments, we use the following architectures.

\begin{table}[H]
\centering
\begin{tabular}{@{}llll@{}}
\toprule
& Adult                      & COMPAS  & German Credit       \\ \midrule
Encoder layers                 & [input dim, 16, 32, 10] & [input dim, 8, 10, 5] & [input dim, 16, 32, 10]  \\
Decoder layers                  & [10, 16, 32, input dim] & [5, 10, 8, input dim] & [10, 16, 32, input dim] \\
Type                     & Fully connected                      & Fully connected & Fully connected \\
Intermediate activations & ReLU                                & ReLU      & ReLU                  \\
Loss function   & MSE                             & MSE  & MSE                 \\ \bottomrule
\end{tabular}
\caption{Autoencoder details}
\label{tab:model_dets}
\end{table}
%

\section{Additional Empirical Evaluation}\label{app:additional_experiments}
\subsection{Remaining Empirical Results from Section \ref{sec:results}}
In Table \ref{tab:rank_german}, we show the remaining results on the German Credit data pertaining to the Spearman rank correlation experiments, while Figure \ref{fig:dmatch_credit_logistic} depicts the remaining $d_{\text{match}}$ results for the German Credit data set on the logistic regression classifier.

\begin{table*}[h!]
\centering
\scalebox{0.85}{
\begin{tabular}{lccccc} 
\toprule  
& \multicolumn{4}{c}{German Credit} \\
\cmidrule(lr){2-5}
& \multicolumn{2}{c}{LR} & \multicolumn{2}{c}{ANN} \\
\cmidrule(lr){2-3} \cmidrule(lr){4-5}
Model &      SCFE &       CCHVAE &  SCFE &       CCHVAE \\ 
\cmidrule(lr){1-1} \cmidrule(lr){2-2} \cmidrule(lr){3-3} \cmidrule(lr){4-4} \cmidrule(lr){5-5} 
CW &  $\mathbf{0.92 \pm   0.04}$ & $0.52 \pm 0.08$ & $\mathbf{0.98 \pm 0.02}$ & $0.72 \pm 0.13$  \\ 
DF &  $\mathbf{0.92 \pm  0.04}$ & $0.57 \pm 0.08$ & $0.97 \pm 0.02$ & $0.72 \pm 0.13$ \\  
NAE &  $0.44 \pm  0.11$ & $\mathbf{0.99 \pm 0.01}$ & $0.71 \pm 0.19$ & $\mathbf{0.99 \pm 0.01}$  \\ 
\bottomrule 
\end{tabular}
}
\caption{Average Spearman rank correlation between counterfactual perturbations and adversarial perturbations. For every input $\bx$, we compute the corresponding adversarial perturbation $\delta_{\text{AE}}$ and the counterfactual perturbation $\delta_{\text{CE}}$. We then compute the rank correlation of $\delta_{\text{AE}}$ and $\delta_{\text{CE}}$ and report their means. The maximum rank correlation is obtained for methods that belong to the same categories (gradient based vs.\ manifold-based).}
\label{tab:rank_german}
\end{table*}

\begin{figure}[!htbp]
\vspace{-0.1in}
\begin{subfigure}[b]{\textwidth}
         \centering
         \includegraphics[scale=0.5]{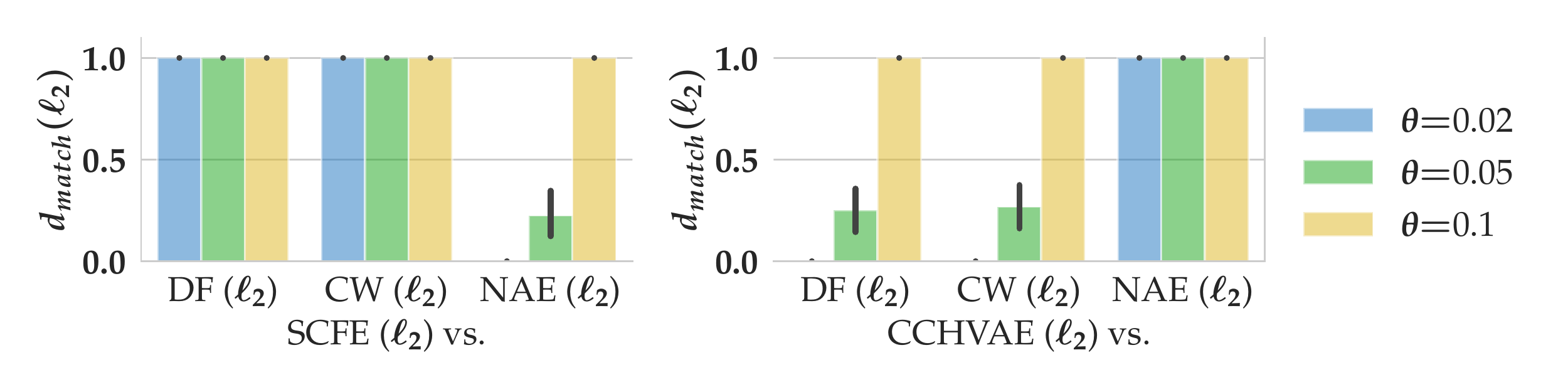}
\end{subfigure}
\caption{Analyzing to what extent different \counterexpmethods and \advexmethods are empirically equivalent for the logistic regression classifier with German Credit data. We compute $d_{\text{match}}$ from \eqref{eq:eval_measures2} with varying thresholds $\theta = \{0.02, 0.05, 0.1\}$. Missing bars indicate that there was no match. 
}
\label{fig:dmatch_credit_logistic}
\end{figure}
We also include results for Neural Networks in Appendix~\ref{app:ann_5pt2}.

\subsection{Empirical Evaluation with ANN}\label{app:ann_5pt2}

\begin{figure}[H]
\begin{subfigure}[b]{\textwidth}
         \centering
         \includegraphics[scale=0.5]{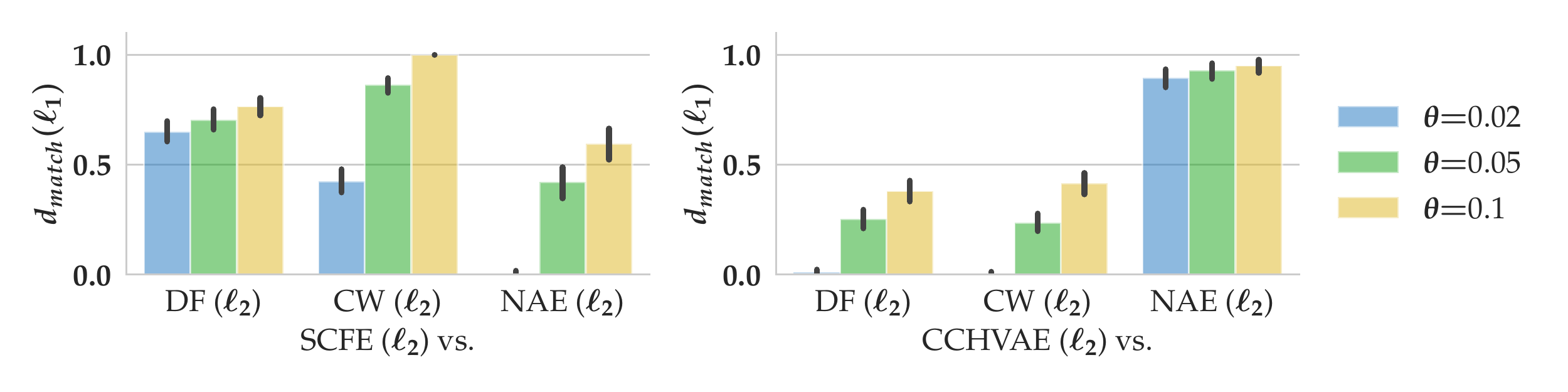}
         \caption{Adult}
         \label{fig:dmatch_adult_ann}
\end{subfigure}
\vfill
\begin{subfigure}[b]{\textwidth}
         \centering
         \includegraphics[scale=0.5]{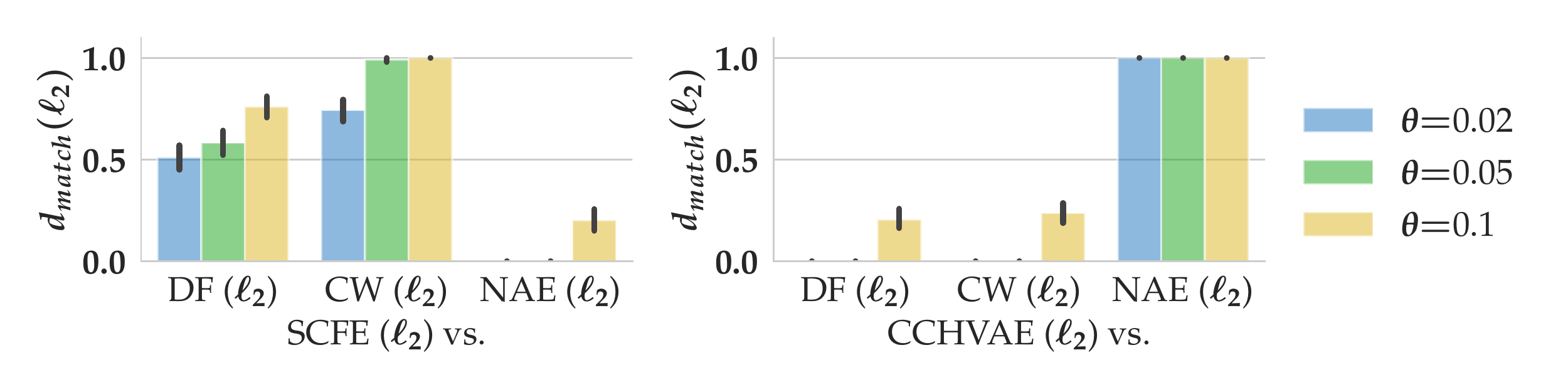}
         \caption{COMPAS}
         \label{fig:dmatch_compas_ann}
\end{subfigure}
\vfill
\begin{subfigure}[b]{\textwidth}
         \centering
         \includegraphics[scale=0.5]{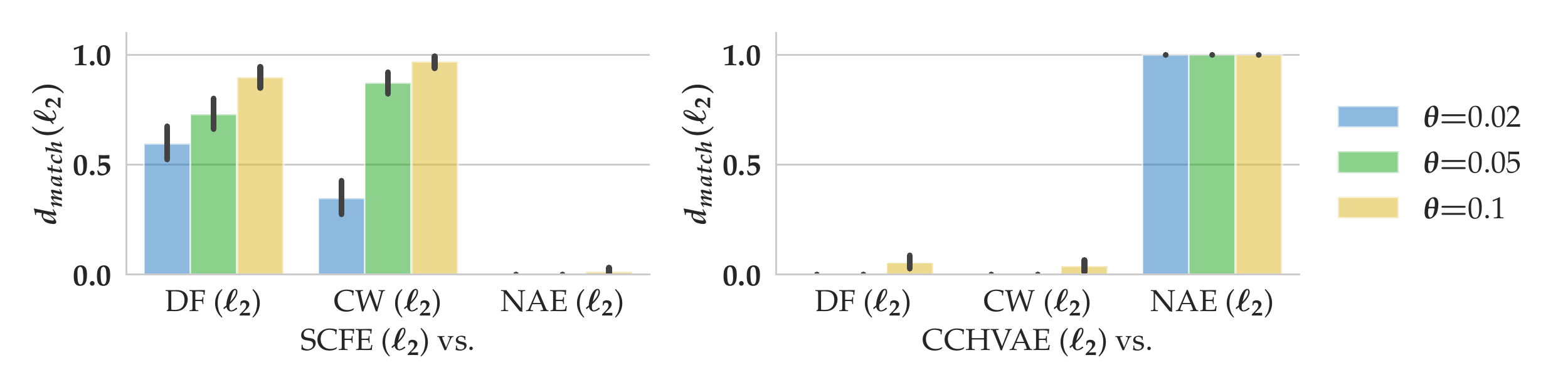}
         \caption{German Credit}
         \label{fig:dmatch_credit_ann}
\end{subfigure}
\caption{Analyzing to what extent different \counterexps and \advexs are empirically equivalent for the 2-layer ANN classifier. To do that, we compute $d_{match}$ from \eqref{eq:eval_measures2} with varying thresholds $\theta = \{0.02, 0.05, 0.1\}$. Missing bars indicate that there was no match.
}
\end{figure}

\begin{figure}[!htbp]
\begin{subfigure}[b]{\textwidth}
         \centering
         \includegraphics[scale=0.5]{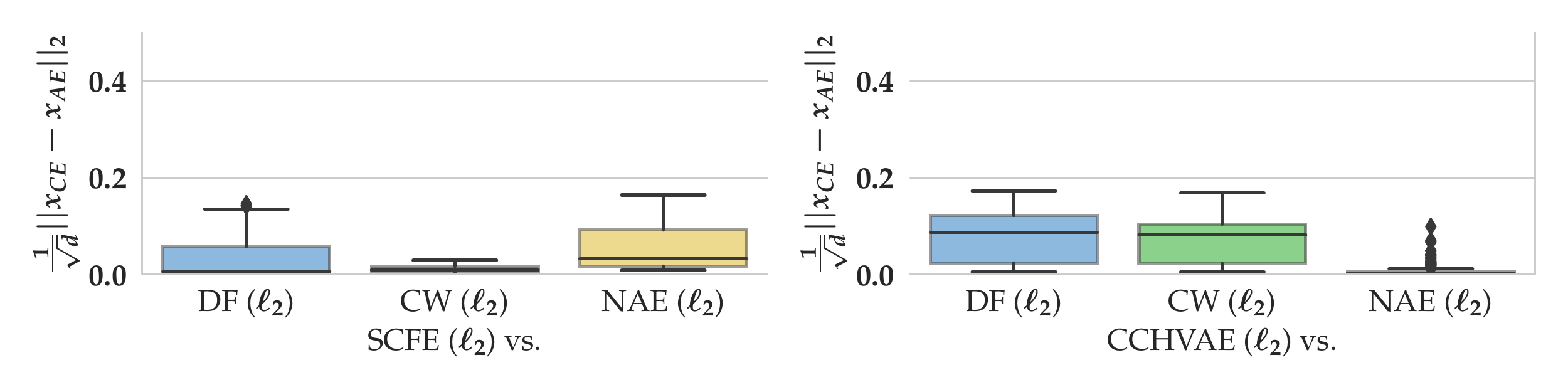}
         \caption{Adult}
         \label{fig:mad_adult_logistic_1n}
\end{subfigure}
\vfill
\begin{subfigure}[b]{\textwidth}
         \centering
         \includegraphics[scale=0.5]{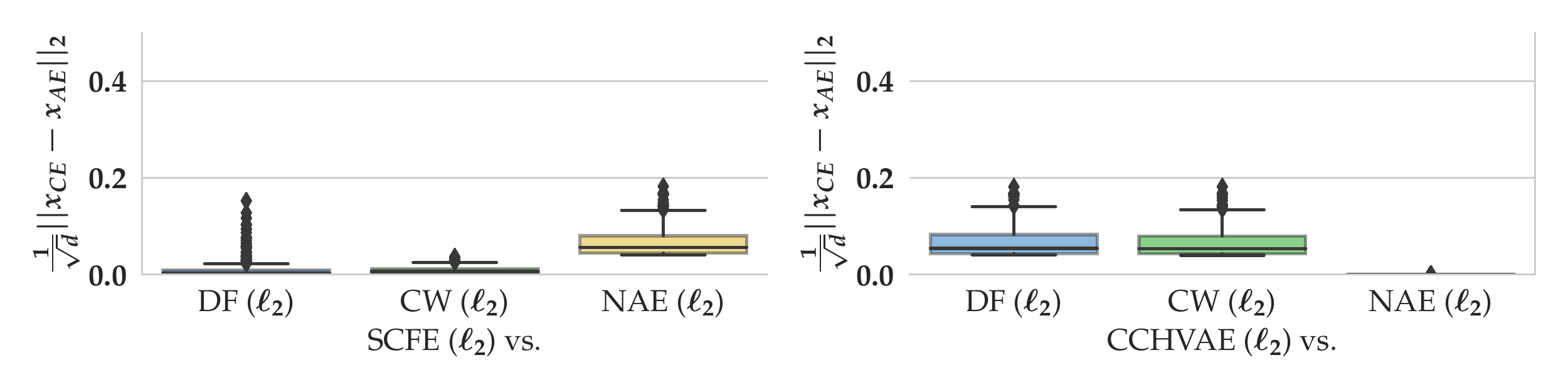}
         \caption{COMPAS}
         \label{fig:mad_compas_logistic_1n}
\end{subfigure}
\vfill
\begin{subfigure}[b]{\textwidth}
         \centering
         \includegraphics[scale=0.5]{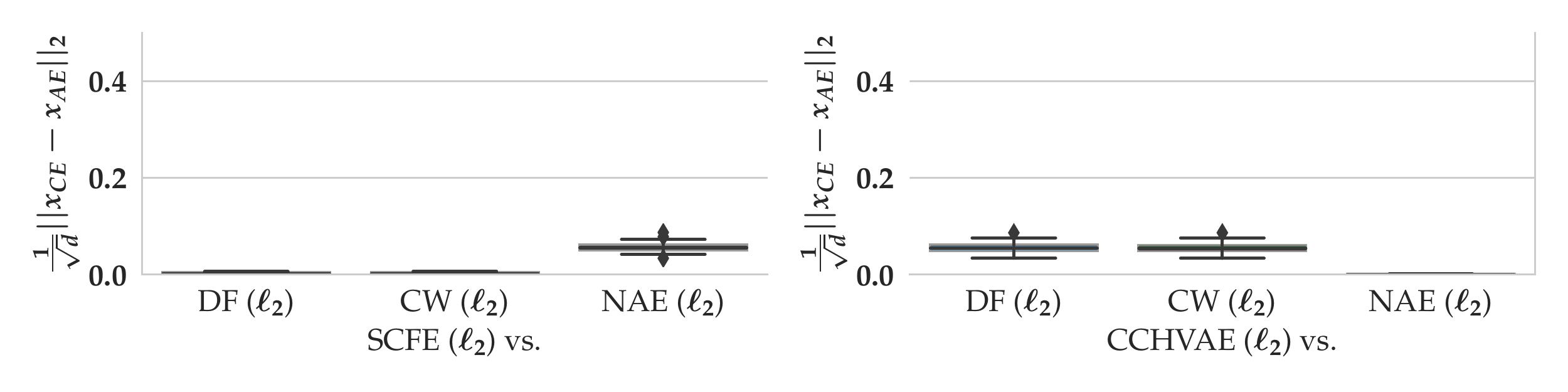}
         \caption{German Credit}
         \label{fig:mad_credit_logistic_1n}
\end{subfigure}
\caption{Distribution of instance wise norm comparisons for the logistic regression model. We show the distribution of cost comparisons across negatively predicted instances ($\hat{y} = 0$) for which we computed \advexs and \counterexps.}
\label{fig:l1_all_logistic}
\end{figure}

\begin{figure}[!htbp]
\begin{subfigure}[b]{\textwidth}
         \centering
         \includegraphics[scale=0.5]{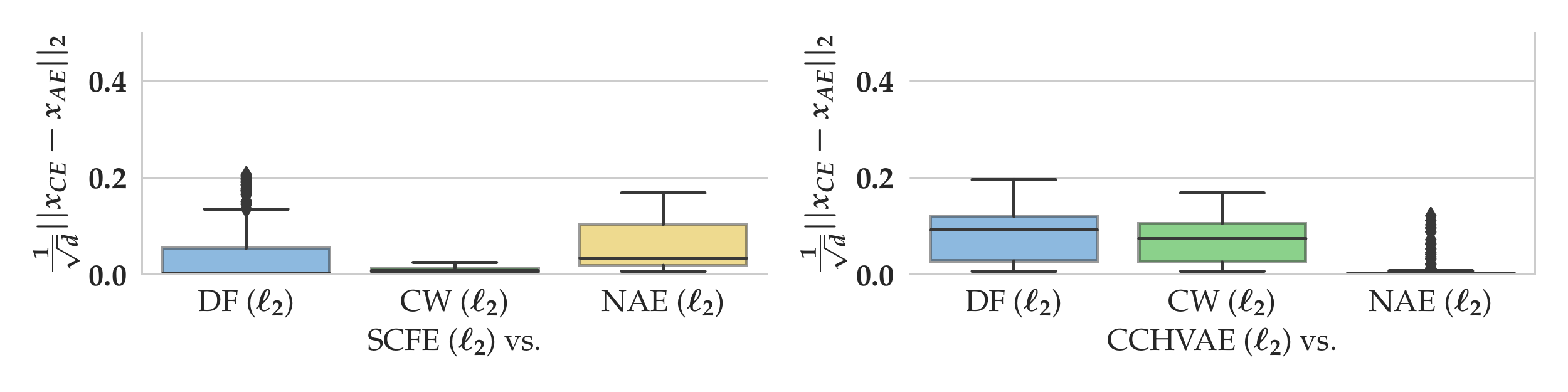}
         \caption{Adult}
         \label{fig:mad_adult_ann_1n}
\end{subfigure}
\vfill
\begin{subfigure}[b]{\textwidth}
         \centering
         \includegraphics[scale=0.5]{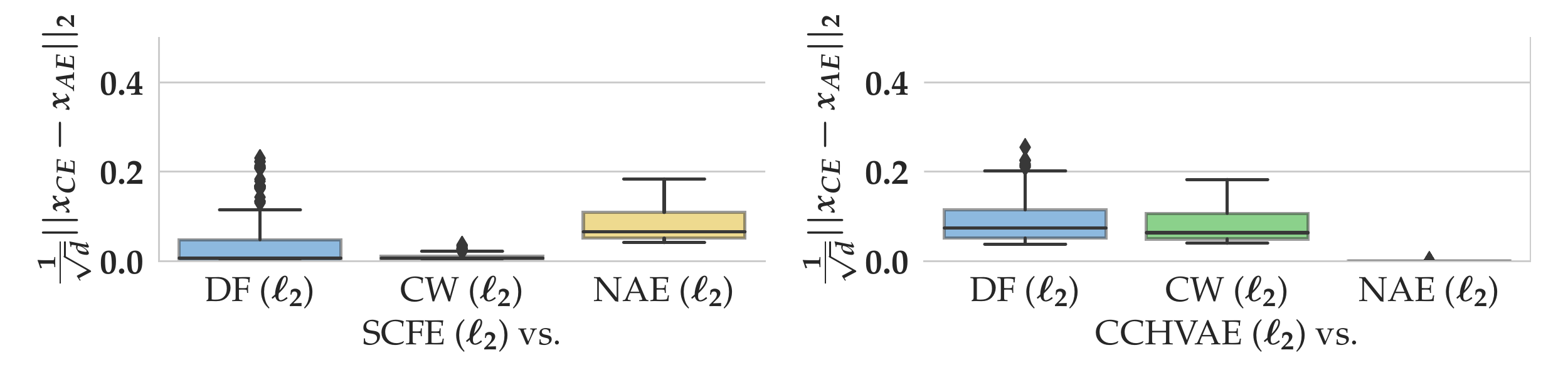}
         \caption{COMPAS}
         \label{fig:mad_compas_ann_1n}
\end{subfigure}
\vfill
\begin{subfigure}[b]{\textwidth}
         \centering
         \includegraphics[scale=0.5]{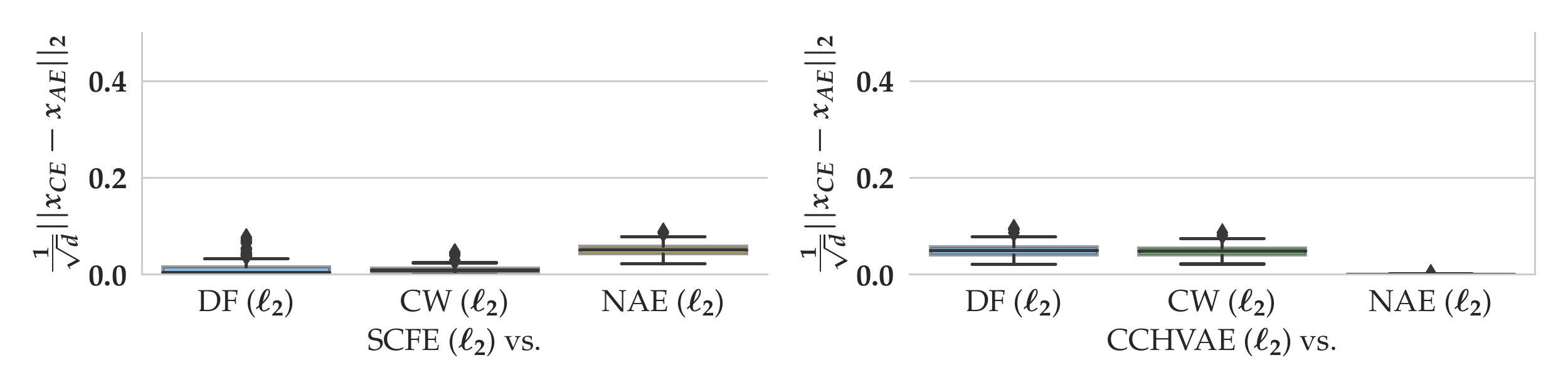}
         \caption{German Credit}
         \label{fig:mad_credit_ann_1n}
\end{subfigure}
\caption{Distribution of instance wise norm comparisons for the 2-layer ANN. We show the distribution of cost comparisons across negatively predicted instances ($\hat{y} = 0$) for which we computed \advexs and \counterexps.}
\label{fig:l1_all_ann}
\end{figure}